\documentclass[journal]{IEEEtran}
\ifCLASSINFOpdf
\else
\fi
\usepackage{cite}
\usepackage{amsmath,amssymb,amsfonts}
\usepackage{algorithmic}
\usepackage{graphicx}
\usepackage{textcomp}

\makeatletter

\let\proof\@undefined
\let\endproof\@undefined
\makeatother

\usepackage{graphics}
\usepackage{epsfig}
\usepackage{mathptmx}
\usepackage{amsmath}
\usepackage{amssymb}
\usepackage{amsthm}
\usepackage{optidef}
\usepackage{tikz}
\usepackage[english]{babel}
\usepackage[utf8]{inputenc}
\usepackage{caption}
\usepackage{subcaption}
\usepackage{lipsum}
\usepackage[mathscr]{euscript}
\usepackage{thmtools}
\usepackage{caption}
\usepackage{subcaption}
\usepackage{stmaryrd}
\usepackage{algorithm}
\usepackage{breqn}
\usepackage{hyperref}
\usepackage{epsfig}
\usepackage{xcolor}

\def\BibTeX{{\rm B\kern-.05em{\sc i\kern-.025em b}\kern-.08em
    T\kern-.1667em\lower.7ex\hbox{E}\kern-.125emX}}

\DeclareMathAlphabet{\mathcal}{OMS}{cmsy}{m}{n}
\newcommand*{\QEDB}{\hfill\ensuremath{\square}}
\newcommand*{\QEDA}{\hfill\ensuremath{\blacksquare}}

\newcommand{\sign}{\text{sign}}

\newtheorem{prop}{Proposition}
\newtheorem{cor}{Corollary}
\newtheorem{defn}{Definition}

\newtheorem{rmk}{Remark}
\newtheorem{theorem}{\textbf{Theorem}}

\begin{document}
%
\title{Control of Mobile Robots Using Barrier Functions Under Temporal Logic Specifications}
%
%
%

\author{Mohit Srinivasan, \IEEEmembership{Student Member, IEEE}, and Samuel Coogan, \IEEEmembership{Member, IEEE}
\thanks{This work was supported in part by DARPA under Grant N66001-17-2-4059 and the National Science Foundation under Grant 1749357}
\thanks{Mohit Srinivasan is with the School of Electrical and Computer Engineering, Georgia Institute of Technology, Atlanta, Georgia 30332 USA (e-mail: mohit.srinivasan@gatech.edu).}
\thanks{Samuel Coogan is with the School of Electrical and Computer Engineering and the School of Civil and Environmental Engineering, Georgia Institute of Technology, Atlanta, Georgia 30332 (email: sam.coogan@gatech.edu)}
\thanks{Manuscript received month day, 20xx; revised month day, 20xx.}}

%
%

\markboth{Journal of \LaTeX\ Class Files,~Vol.~14, No.~8, August~2015}%
{Shell \MakeLowercase{\textit{et al.}}: Bare Demo of IEEEtran.cls for IEEE Journals}
%



\maketitle


\begin{abstract}
In this paper, we propose a framework for the control of mobile robots subject to temporal logic specifications using barrier functions. Complex task specifications can be conveniently encoded using linear temporal logic. In particular, we consider a fragment of linear temporal logic which encompasses a large class of motion planning specifications for a robotic system. Control barrier functions have recently emerged as a convenient tool to guarantee reachability and safety for a system. In addition, they can be encoded as affine constraints in a quadratic program. 
In this paper, a fully automatic framework which translates a user defined specification in temporal logic to a sequence of barrier function based quadratic programs is presented. In addition, with the aim of alleviating infeasibility scenarios,
we propose methods for composition of barrier functions as well as a prioritization based control method to guarantee feasibility of the controller. We prove that the resulting system trajectory synthesized by the proposed controller satisfies the given specification. Robotic simulation and experimental results are provided in addition to the theoretical framework.
\end{abstract}


%
\IEEEpeerreviewmaketitle

\section{Introduction}
\label{sec:introduction}
System specifications to be satisfied by mobile robotic systems are increasing in complexity. For example, motion planning for systems such as robotic manipulators \cite{murray_manipulators}, personal assistants \cite{pers_assistants}, and quadrotors \cite{7989375} involves complex specifications to be satisfied by the system. Safety critical systems such as the power grid \cite{DHS} and automation floors \cite{simon_2019} rely on distributed controllers in order to function in the desired manner. These controllers are again tasked with satisfying complex specifications. Hence, failure of these controllers can lead to a collapse of the safety critical infrastructure \cite{SCI}. To that end, synthesizing controllers with formal guarantees on their correct functioning is of key importance.

In this paper, we present a control architecture for the control of mobile robotic systems subject to linear temporal logic specifications using control barrier functions, which addresses some of the challenges in the previously discussed applications. In particular, we address the issue of situations where proposed methods in existing literature can render the controller infeasible. With the synthesis of the controller, we then shift focus towards providing formal guarantees regarding the proposed controller framework. In particular, we prove that the system trajectory generated by the proposed controller satisfies the given specification.

\subsection{Background}
Barrier functions were first introduced in optimization. A historical account of their use can be found in Chapter 3 in \cite{forsgren2002interior}. Usage of barrier functions is now common throughout the control, verification and robotics literature due to their natural relationship with Lyapunov-like functions. Control barrier function (CBF) based quadratic programs (QPs) were first used in \cite{ames2014control}, \cite{xu2016robustness} in the context of automotive applications such as adaptive cruise control (ACC). Recently, control barrier functions have been used in the context of multi-agent systems to guarantee collision avoidance between robots \cite{7989375}, \cite{7857061}, \cite{7526486}. Given a minimum distance to be maintained between the robots, the safety set is encoded as the super zero level set of a zeroing control barrier function (ZCBF) \cite{xu2016robustness}. The authors then use a QP based controller with the ZCBFs as affine constraints in order to guarantee forward invariance of this safety set. This in turn implies that the robots never collide. Such a framework has also been applied to quadrotors \cite{7989375} where the safety set is considered to be a super ellipsoid which allows quadrotors to avoid collisions with each other.

ZCBFs guarantee asymptotic convergence to desired sets \cite{xu2016robustness}. However, since we focus on motion planning specifications, we require finite time reachability guarantees. Recently, \cite{li2018formally}, \cite{mohit_cdc18} have introduced finite time control barrier functions for finite time reachability specifications. In \cite{li2018formally}, finite time barrier functions were used to achieve smooth transitions between different behaviors in a multi-agent system. The key objective in \cite{li2018formally} was to ensure composability of different formation behaviors by making sure that the multi-agent communication graph is appropriate for the next desired formation, whereas in \cite{mohit_cdc18}, a method for the composition of multiple finite time barrier functions was introduced. Barrier functions have also been introduced in  hybrid systems theory \cite{sanfelice} to guarantee forward invariance of hybrid inclusions.

Finite and infinite horizon specifications which are useful for mobile robotic systems can be conveniently encoded using linear temporal logic (LTL). The power of LTL originates from the wealth of tools available in the model checking literature \cite{baier2008principles} which allows for generating trajectories for the robots given a specification in temporal logic. LTL based control of robotic systems has been well studied and standard methods first create a finite abstraction of the original dynamical system \cite{alur_disc_hybsys}, \cite{belta_rect}, \cite{fully_auto_belta}, \cite{receding_topcu}. This abstraction can informally be viewed as a labeled graph that represents possible behaviors of the system. Given such a finite abstraction, controllers can be automatically constructed using an automata-based approach \cite{fully_auto_belta}, \cite{baier2008principles}, \cite{vardi_complex_goals}, \cite{fainekos2009temporal}. However, abstracting the state space is computationally expensive especially with complex system dynamics and specifications.

In our framework, we avoid the difficulties associated with computation of any automaton from the specification or a discretization of the state space. Since CBFs can be conveniently encoded within a QP, the controller is amenable to real time implementations without the need for an abstraction of the state space or the system dynamics. Other authors have explored discretization free techniques as well. The authors in \cite{lindemann2019control} discuss the use of time varying control barrier functions for signal temporal logic tasks (STL). 
In \cite{lindemann_coupled}, the authors use time-varying barrier functions for control of coupled multi-agent systems subject to STL tasks. In both \cite{lindemann2019control} and \cite{lindemann_coupled}, the authors do not allow for repetitive tasks, a specification which can be captured by our proposed framework. The authors in \cite{Raman}, \cite{belta2019formal}, \cite{Liu} discuss control methods for STL tasks. However, the methods proposed result in computationally expensive mixed integer linear programs. Control methods in the discrete time non-deterministic setting have been explored by \cite{Farhani_shrinking}. Learning based frameworks are discussed by the authors in \cite{aksaray_belta_q_learning}, \cite{muniraj2018enforcing}, \cite{dimos_control_guided}. Control techniques for continuous-time multi-agent systems given fragment of STL tasks has been presented in \cite{lindemann2018decentralized}. The authors in \cite{pant_mangharam_drones} discuss a similar continuous time method. However, a non-convex optimization problem may have to be solved.

\subsection{Contributions}
There are two primary contributions of this work. First, we propose a barrier function based controller framework to synthesize system trajectories that satisfy a given user defined specification. In particular, the proposed framework automatically translates the user defined specification formalized in a subset of linear temporal logic (LTL) to a sequence of barrier function based quadratic programs. The approach adopted in this paper is a discretization free approach which alleviates some of the computational issues arising from abstraction based control of mobile robots \cite{alur_disc_hybsys}, \cite{belta_rect}, \cite{fully_auto_belta}, \cite{receding_topcu}, \cite{vardi_complex_goals}, \cite{fainekos2009temporal}. Then, we provide formal guarantees that the proposed controller framework produces a system trajectory that satisfies the given specification. The proposed family of LTL specifications in our work can capture more complex specifications than the fragment considered in \cite{lindemann2019control, lindemann_coupled}. In addition, our guarantees are different from other papers on temporal logic based control using barrier functions \cite{lindemann2019control, lindemann_coupled} in that we characterize the family of trajectories that satisfy the given specification, and then prove that the proposed controller indeed produces a trajectory that belongs to the set of satisfying trajectories. The trajectory generated by the proposed CBF based controller is analyzed and the guarantees of CBFs translate to guarantees on the system trajectory.

Second, we address the issue of controller infeasibility for both reachability and safety objectives, a common difficulty in CBF based real-time control \cite{high_rel_deg_nguyen}, \cite{xiao2019control_rel_deg}. We first illustrate a scenario where the method of encoding multiple finite time reachability objectives individually in a CBF based QP framework---as proposed in, \emph{e.g.}, \cite{li2018formally}---fails. This is because encoding each reachability specification as a separate constraint in the QP is restrictive when the reachability objective is defined by multiple CBFs. 
We instead propose a method to compose multiple finite time control barrier functions as a single QP constraint which results in a larger feasible set while retaining the same guarantees as those established in \cite{li2018formally}. Next, 
we consider the case when some safety specifications (captured using ZCBFs) are in conflict with others and propose a prioritization scheme, similar to the method discussed in \cite{notomista2019optimal}, in order to relax the ZCBFs which characterize the safety constraints. We guarantee satisfaction of the reachability tasks while minimally violating the safety specifications. While not a main focus of this paper, this relaxation is a novel reformulation of ideas presented in \cite{xu2016robustness} and \cite{notomista2019optimal} for use in motion planning problems with finite time reachability constraints.


A preliminary version of this work was presented in our conference paper \cite{mohit_cdc18} where we formulated the notion of composition of multiple finite time control barrier functions. In the present paper, we significantly extend those results in order to synthesize an automated framework (full solution) to transition from a specification belonging to a fragment of LTL, to the barrier function based controller.

This paper is organized as follows. Section~\ref{sec:math_background} introduces control barrier functions, linear temporal logic and the quadratic program based controller used for trajectory generation of the system. In Section~\ref{sec:prob_stat}, we discuss the problem statement that is addressed in this paper. Section~\ref{sec:QP_Feasibility} introduces the idea of composite finite time control barrier functions \cite{mohit_cdc18} and the prioritization scheme for different zeroing barrier functions. In Section~\ref{sec:formal_guar}, we propose the QP based controller and develop the theoretical framework which provides a formal guarantee that the proposed controller synthesizes a system trajectory that satisfies the given specification. Section~\ref{sec:case_study} discusses a multi-agent system case study with simulation and experimental results. Section~\ref{sec:conclusion} provides concluding remarks.

\section{Mathematical Background}
\label{sec:math_background}
In this section, we provide background on control barrier functions (CBFs) and the guarantees on invariance and reachability of sets obtained from them, linear temporal logic (LTL) which is the specification language, and the quadratic program (QP) based controller with the CBFs as constraints which will be used to synthesize the trajectory for a control-affine robotic system.

\subsection{Control Barrier Functions (CBFs)}
Consider a continuous time, control-affine dynamical system
\begin{equation}
\label{ControlAffineSystem}
\dot x = f(x) + g(x)u \text{ ,}
\end{equation}
where $f: \mathcal{X} \rightarrow \mathbb{R}^n$ and $g: \mathcal{X} \rightarrow \mathbb{R}^{n \times m}$ are locally Lipschitz continuous, $x \in \mathcal{X}\subseteq \mathbb{R}^{n}$ is the state of the system, and $u \in \mathbb{R}^{m}$ is the control input applied to the system.

Before we introduce the notion of control barrier functions, we define an extended class $\mathcal{K}$ function \cite{khalil2002nonlinear} $\alpha : \mathbb{R} \rightarrow \mathbb{R}$ as a function that is strictly increasing and $\alpha(0) = 0$.

\begin{defn}[Zeroing Control Barrier Function (ZCBF)]
\label{def:zcbf}
A continuously differentiable function $h :\mathcal{X}\to\mathbb{R}$ is a \emph{zeroing control barrier function (ZCBF)} if there exists a locally Lipschitz extended class $\mathcal{K}$ function $\alpha$ such that for all $x \in \mathcal{X}$,
\begin{equation}
\label{eq:zcbf_sup}
\sup_{u\in \mathbb{R}^{m}} \bigg\{ L_{f}h(x) + L_{g}h(x)u + \alpha(h(x)) \bigg\} \geq 0
\end{equation}
where $L_{f} h(x) = \frac{\partial h(x)}{\partial x} f(x)$ and $L_{g} h(x) = \frac{\partial h(x)}{\partial x} g(x)$ are the Lie derivatives of $h$ along $f$ and $g$ respectively. \QEDB
\end{defn}

Let $\Sigma \subseteq \mathcal{X}$ be a safety set defined as $\Sigma = \{ x \in \mathcal{X} \mid h(x) \geq 0\}$ where $h : \mathcal{X} \rightarrow \mathbb{R}$ is a ZCBF. The set of control inputs that satisfy \eqref{eq:zcbf_sup} at any given state $x \in \mathcal{X}$ is then defined as
\begin{equation}
    \label{eq:zcbf_controls}
    \mathcal{U}_{\Sigma}(x) = \bigg\{ u \in \mathbb{R}^{m} \mid L_{f}h(x) + L_{g}h(x)u + \alpha(h(x)) \geq 0 \bigg\}.
\end{equation}

One can guarantee forward invariance of desired sets under the existence of a suitable ZCBF as formalized in the following proposition.

\begin{prop}[Corollary 1, \cite{xu2016robustness}]
\label{prop:zcbf}
If $h$ is a ZCBF, then any continuous feedback controller satisfying $u \in \mathcal{U}_{\Sigma}(x)$ renders the set $\Sigma$ forward invariant for the system \eqref{ControlAffineSystem}. \QEDA
\end{prop}

We now define finite time convergence control barrier functions, first introduced in \cite{li2018formally}, which guarantee finite time convergence to desired sets in the state space.

\begin{defn}[Finite Time Convergence Control Barrier Function (FCBF)]
\label{def:fcbf}
A continuously differentiable function $h :\mathcal{X}\to\mathbb{R}$ is a \emph{finite time convergence control barrier function} if there exist parameters $\rho \in [0, 1)$ and $\gamma > 0$ such that for all $x \in \mathcal{X}$,
\begin{equation}
\label{eq:fcbf_sup}
\sup_{u\in \mathbb{R}^{m}} \left\{ L_{f}h(x) + L_{g}h(x)u + \gamma \cdot \text{sign}(h(x)) \cdot |h(x)|^{\rho} \right\} \geq 0
\end{equation}
where $L_{f} h(x) = \frac{\partial h(x)}{\partial x} f(x)$ and $L_{g} h(x) = \frac{\partial h(x)}{\partial x} g(x)$. \QEDB
\end{defn}

Let $\Gamma \subseteq \mathcal{X}$ be a target set defined as $\Gamma = \{ x \in \mathcal{X} \mid h(x) \geq 0\}$ where $h : \mathcal{X} \rightarrow \mathbb{R}$. Let the set of control inputs that satisfy \eqref{eq:fcbf_sup} at any state $x \in \mathcal{X}$ be given by
\begin{equation}
    \label{eq:fcbf_controls}
    \mathcal{U}_{\Gamma}(x) = \bigg\{ u \in \mathbb{R}^{m} \bigg| L_{f}h(x) + L_{g}h(x)u + \gamma \cdot \text{sign}(h(x)) \cdot |h(x)|^{\rho} \geq 0 \bigg\}
\end{equation}

If $h$ is a FCBF, then there exists a control input $u$ that drives the state of the system $x$ to the target set $\{x \in \mathcal{X} \mid h(x)\geq 0\}$ in finite time, as formalized next.

\begin{prop}[Proposition III.1, \cite{li2018formally}]
\label{prop:fcbf}
If $h$ is a FCBF for \eqref{ControlAffineSystem}, then, for any initial condition $x_{0} \in \mathcal{X}$ and any continuous feedback control $u : \mathcal{X} \rightarrow \mathbb{R}^{m}$ satisfying $u \in \mathcal{U}_{\Gamma}(x) $ for all $x \in \mathcal{X}$, the system will be driven to the set $\Gamma$ in a finite time $0 < T < \infty$ such that $x(T)\in \Gamma$, where the time bound is given by $T = \frac{|h(x_0)|^{1 - \rho}}{\gamma (1 - \rho)}$. Moreover, $\Gamma$ is forward invariant so that the system remains in $\Gamma$ for all $t \geq T$. \QEDA
\end{prop}

ZCBFs and FCBFs will form the basis for our control synthesis methodology. Next, we discuss the temporal language used to specify complex robotic system specifications in our framework.

\subsection{Linear Temporal Logic}
Complex and rich system properties can be expressed succinctly using linear temporal logic (LTL). The power of LTL lies in the wealth of tools available in the model checking literature \cite{baier2008principles} which can be leveraged for the synthesis of controllers in the continuous domain. LTL formulas are developed using atomic propositions which label regions of interest within the state space. These formulas are built using a specific grammar. LTL formulas without the next operator are given by the following grammar \cite{baier2008principles}:
\begin{equation}
\phi = \pi | \neg \phi | \phi \vee \phi | \phi \mathcal{U} \phi
\end{equation}
where $\pi$ is a member of the set of atomic propositions denoted by $\Pi$, and $\phi$ is a propositional formula that represents an LTL specification. Since we deal with continuous time systems in this work, the use of the ``next" operator ($\bigcirc$) lacks meaningful interpretation, and hence, this operator is not included in our framework. Nonetheless, a large class of motion planning specifications (for example, the class of specifications proposed in \cite{guo_ltl_planning} and characterized below in Definition \ref{def:LTLfrag} do not require the next operator. In particular, finite time reachability specifications can be encoded in our proposed work.

We use the standard graphical notation for the temporal operators including $\Box$ (``Always"), $\Diamond$ (``Eventually"), $\Diamond \Box$ (``Persistence") and $\Box \Diamond$ (``Recurrence"). From the negation ($\neg$) and the disjunction ($\vee$) operators, we can define the conjunction ($\wedge$), implication ($\rightarrow$), and equivalence ($\leftrightarrow$) operators. We can thus derive for example, the eventually ($\Diamond$) and always ($\Box$) operators as $\Diamond \phi = \top \mathcal{U} \phi$ and $\Box \phi = \neg \Diamond \neg \phi$ respectively. Below we provide informal interpretations of these operators with respect to an $LTL$ formula $\phi$.
\begin{itemize}
\item $\Diamond \phi$ is satisfied if $\phi$ is satisfied sometime in the future. That is, $\phi$ is satisfied at some point of time in the future.
\item $\Box \phi$ is satisfied if $\phi$ is satisfied for all time. That is, $\phi$ is satisfied for all time.
\item $\Diamond \Box \phi$ is satisfied if $\phi$ becomes satisfied at some point of time in the future and then remains satisfied for all time.
\item $\Box \Diamond \phi$ is satisfied if $\phi$ is satisfied infinitely often at various points of time in the future.
\end{itemize}

Next we discuss the QP based controller which will be used for the synthesis of the system trajectory.

\subsection{Quadratic Program (QP) based controller}
Given a FCBF or ZCBF $h$, the constraints \eqref{eq:zcbf_controls} and \eqref{eq:fcbf_controls} are affine in the control input $u$, and hence they can be conveniently encoded as affine constraints in a QP. Hence this formulation is amenable to efficient online computation of feasible control inputs. In particular, for fixed $x \in \mathcal{X}$, the requirement that $u \in \mathcal{U}_{\Gamma}(x)$ and/or $u \in \mathcal{U}_{\Sigma}(x)$ becomes a linear constraint and we define a minimum energy quadratic program (QP) as
\begin{equation}
\begin{aligned}
\label{intro_qp}
& \underset{u \in \mathbb{R}^{m}}{\text{min}}
\quad ||u||_{2}^{2}\\
& \text{s.t \quad} u \in \mathcal{U}_{\Gamma}(x) \text{ and/or } u \in \mathcal{U}_{\Sigma}(x) \text{.}
\end{aligned}
\end{equation}

We note that $\eqref{intro_qp}$ can encode both finite time reachability as well as forward invariance requirements as constraints in the QP. This QP when solved returns the pointwise minimum norm control law that drives the system to the goal set $\Gamma$ in finite time and/or guarantees invariance of $\Sigma$. We will reference this idea of a QP based controller throughout this paper in the context of our theoretical framework.

\begin{rmk}
We note that multiple ZCBFs and multiple FCBFs can be encoded as separate constraints in the QP. In this case, we solve a single QP with multiple barrier function constraints. For example, see \cite{xu2016robustness}, \cite{li2018formally}.
\end{rmk}

\section{Problem Statement}
\label{sec:prob_stat}
In this paper, we consider a continuous time mobile robotic system in control affine form as in \eqref{ControlAffineSystem}.
We assume that $\mathcal{X}$ contains regions of interest which are labeled by a set of atomic propositions $\Pi = \{\pi_1, \pi_2, \pi_3, \dots, \pi_n\}$ with the labeling function $L : \mathcal{X} \rightarrow 2^\Pi$ so that $\pi \in \Pi$ is true at $x \in \mathcal{X}$ if and only if $\pi \in L(x)$. These regions may overlap and need not constitute a partition or cover of $\mathcal{X}$. For each $\sigma \in 2^{\Pi}$, we have $L^{-1}(\sigma) = \{ x \in \mathcal{X} | \sigma = L(x) \}$. Let $\Pi_{aug} = \{ \pi_1, \pi_2, \dots, \pi_n, \overline{\pi_1}, \overline{\pi_2}, \dots, \overline{\pi_n}\}$ be the augmented set of atomic propositions where we define $\overline{\pi_i} = \neg \pi_i$ for all $i \in \{ 1,2,\dots, n\}$. The set $\Pi_{aug}$ is also called the set of \textit{literals} \cite{baier2008principles}. Thus, we identify $\neg\overline{\pi}_i=\pi_i$ for all $i \in \{ 1,2,\dots, n\}$. In addition, define
\begin{equation}
\label{def:S_aug}
\begin{split}
S(\Pi_{aug}) = \{ J \subset \Pi_{aug} \mid &\pi \in J \implies \neg \pi \not \in J \text{ for all } \pi \in \Pi_{aug}\}
\end{split}
\end{equation}
\begin{equation}
\begin{split}
\label{def:P_aug}
P(\Pi_{aug}) = \{ J \subset \Pi_{aug} \mid (\pi_i \in J) \oplus & (\overline{\pi_i} \in J) \\
&\text{ for all } i\in \{ 1,2,\dots, n\}\}
\end{split}
\end{equation}
where $\oplus$ is the exclusive disjunction operator. Observe that $P(\Pi_{aug}) \subset S(\Pi_{aug})$.
A subset of $\Pi_{aug}$ belongs to the family $S(\Pi_{aug})$ if it does not contain an atomic proposition and its negation simultaneously, and it further belongs to $P(\Pi_{aug})$ if it contains each atomic proposition exclusive-or its negation.

We consider a fragment of LTL, denoted by $LTL_{robotic}$, which is a modification of the fragment considered in \cite{wolff2013efficient}. Our proposed fragment covers a large class of motion planning tasks, such as the ones discussed in \cite{guo_ltl_planning}, expected from a robotic system.

\begin{defn}[Fragment of LTL]
\label{def:LTLfrag}
The fragment $LTL_{robotic}$ is defined as the class of LTL specifications of the form
\begin{align}
\label{spec}
\phi = \phi_{globe} \wedge \phi_{reach} \wedge \phi_{rec} \wedge \phi_{act}
\end{align}
where $\phi_{globe} = \Box \psi_1$, $\phi_{reach} = \bigwedge\limits_{j \in \mathcal{I}_{2}} \Diamond \psi_{2}^{j}$, $\phi_{rec} = \bigwedge\limits_{j \in \mathcal{I}_{3}} \Box \Diamond \psi_{3}^{j}$ and $\phi_{act} = \Diamond \Box \psi_{4}$.
Here $\mathcal{I}_{2}$ and $\mathcal{I}_{3}$ are finite index sets and $\psi_{1}$, $\psi_{2}^{j}$ for all $j$, $\psi_{3}^{j}$ for all $j$ and $\psi_{4}$ are propositional formulas of the form $\psi_{i} = \bigwedge\limits_{\forall \pi \in J_i} \pi$ with $J_i \in S(\Pi_{aug})$ for all $i \in \{ 1,4\}$, $\psi_{i}^{j} = \bigwedge\limits_{\forall \pi \in J_{i}^{j}} \pi$ with $J_{i}^{j} \in S(\Pi_{aug})$ for all $i \in \{ 2,3\}$ and for all $j \in \mathcal{I}_{i}$. \QEDB
\end{defn}
Below we provide informal definitions of the specifications appearing in the above definition.
\begin{itemize}
    \item $\phi_{\text{globe}}$: This type of specification captures properties that must hold throughout the execution of the system. For example, collision avoidance with obstacles must hold at all times when a robot is navigating in the workspace.
    \item $\phi_{\text{reach}}$: Specifications of this form capture finite time reachability requirements for the system. For example, a robot must reach a region of interest within a finite time. 
    \item $\phi_{\text{rec}}$: This recurrence specification captures, for instance, scenarios where the system must visit regions infinitely often. For example, a robot must visit room A and room B infinitely often.
    \item $\phi_{\text{act}}$: This type of specification captures persistence requirements. For example, a robot must reach a region and then stay within the region for all time.
\end{itemize}

As compared to \cite{wolff2013efficient}, we additionally incorporate reachability specifications ($\Diamond$) without increasing the system complexity due to the abstraction free nature of our proposed framework. We do not include response-to-environment specifications ($\Box(A \implies \Diamond B)$) since time-varying or reactive system specifications are not considered in the context of our proposed barrier function framework. With regard to other widely used fragments, our proposed fragment allows for persistence ($\Diamond \Box$) which cannot be expressed by the Generalized Reactivity (GR(1)) fragment or computation tree logic (CTL) \cite{hadas_reactive_planning}. Our proposed fragment also allows for repetitive tasks ($\phi_{\text{act}}$ and $\phi_{\text{rec}}$) which cannot be captured by the fragment considered in very recent work on barrier function based control using temporal logic \cite{lindemann2019control, lindemann_coupled}.

For any propositional formula $\psi$ omitting temporal operators (\emph{e.g.}, a conjunction of literals) we define the following.
\begin{defn}[Proposition Set]
\label{propset}
The \emph{proposition set} for a propositional formula $\psi$, denoted $\llbracket \psi \rrbracket$,  is the set of all states that satisfy $\psi$. That is, 
\begin{align}
\llbracket \psi \rrbracket = \{ x \in \mathcal{X} \mid L(x) \models \psi\}
\end{align}
where $L(x) \models \psi$ signifies that $\psi$ is true under the evaluation for which all and only propositions in $L(x)\subset \Pi$ are true. \QEDB
\end{defn}

We assume that for each atomic proposition $\pi \in \Pi$, there exists a continuously differentiable function $h : \mathcal{X} \rightarrow \mathbb{R}$ such that $\llbracket \pi \rrbracket = \{ x \in \mathcal{X} | h_{\pi}(x) \geq 0\}$. In this paper, similar to the assumption in \cite{xu2016robustness}, we assume that $L_{g}h_{\pi}(x) \neq 0$ for all $x \in \mathcal{X}$. We ignore the measure-zero set $\{x \in \mathcal{X} | h_\pi(x)=0\}$, and identify $\llbracket \overline{\pi} \rrbracket = \{x \in \mathcal{X} | h_\pi(x) < 0\}$ for each $\pi\in\Pi$. Thus we define $h_{\overline{\pi}}(x) = - h_{\pi}(x)$ for all $\pi \in \Pi$.



The fragment $LTL_{robotic}$ encompasses a class of specifications which cover properties such as finite time reachability, persistence, recurrence, and invariance. These properties are useful to express a number of common robotic system specifications. 

Recall that for any $\sigma \in 2^{\Pi}$, $L^{-1}(\sigma) = \{ x \in \mathcal{X} | \sigma = L(x)\}$. We define a \textit{trace} as a sequence of sets of atomic propositions. The \textit{trace of a trajectory} $x(t)$ of a continuous time dynamical system is defined as the sequence of propositions satisfied by the trajectory. This is formalized in the definition below.

\begin{defn}[Trace of a trajectory \cite{wongpiromsarn2016automata}]
An infinite sequence $\sigma = \sigma_{0}\sigma_{1}\dots$ where $\sigma_{i} \subseteq {\Pi}$ for all $i \in \mathbb{N}$ is the {\normalfont trace of a trajectory} $x(t)$ if there exists an associated sequence $t_{0}t_{1}t_{2}\dots$ of time instances such that $t_{0} = 0$, $t_{k} \rightarrow \infty$ as $k \rightarrow \infty$ and for each $m \in \mathbb{N}$, $t_{m} \in \mathbb{R}_{\geq 0}$ satisfies the following conditions:
\begin{itemize}
\item $t_{m} < t_{m+1}$
\item $x(t_{m}) \in L^{-1}(\sigma_m)$
\item If $\sigma_{m} \neq \sigma_{m+1}$, then for some $t_{m}^{'} \in [t_{m}, t_{m+1}]$, $x(t) \in L^{-1}(\sigma_m)$ for all $t \in (t_{m}, t_{m}')$, $x(t) \in L^{-1}(\sigma_{m+1})$ for all $t \in (t_{m}', t_{m+1})$, and either $x(t_{m}') \in L^{-1}(\sigma_m)$ or $x(t_{m}') \in L^{-1}(\sigma_{m+1})$.
\item If $\sigma_m=\sigma_{m+1}$ for some $m$, then $\sigma_m=\sigma_{m+k}$ for all $k>0$ and $x(t)\in L^{-1}(\sigma_m)$ for all $t \geq t_m$. \QEDB
\end{itemize}
\end{defn}

The last condition of the above definition implies that a trace contains a repeated set of atomic propositions only if this set is then repeated infinitely often. This is useful to capture for example, a stability condition of the system. By forbidding repetitions in other cases, we ensure that a particular trajectory possesses a unique trace. This exclusion is without loss of generality since we only consider $LTL_{robotic}$ specifications without the next operator.

We now define the problem statement that is addressed in this paper.

\newtheorem*{probstat*}{\textbf{Problem Statement}}
\begin{probstat*}
Given a specification in $LTL_{robotic}$ as in \eqref{spec} which is to be satisfied by a mobile robotic system with dynamics as in \eqref{ControlAffineSystem}, synthesize a point-wise minimum norm controller as in \eqref{intro_qp} which produces a system trajectory whose trace satisfies the given specification.
\end{probstat*}

In this paper, we are interested in generating system trajectories using controllers of the form \eqref{intro_qp}, which guarantee satisfaction of the given $\text{LTL}_{\text{robotic}}$ specification. As a secondary objective, we are interested in achieving the stated task by expending minimum power, point-wise in time. These objectives are captured by the problem statement.

Before we detail the theoretical framework to address the above problem statement, we discuss scenarios where the QP based controller \eqref{intro_qp} could be infeasible and present approaches to alleviate infeasibility. This is important since infeasibility can lead to violation of the LTL specification.

\section{Feasibility of QP based Controller}
\label{sec:QP_Feasibility}
Given a specification $\phi$ in $LTL_{robotic}$, in this section, we focus on scenarios where using existing methods in literature \cite{xu2016robustness}, \cite{li2018formally}, \cite{7782377} will render the controller infeasible, and provide solutions for the same. Subsection A discusses Theorem~\ref{thm:comp_fcbf} which appeared in our conference paper \cite{mohit_cdc18}, while subsection B proposes a relaxed formulation of the QP based controller.

\subsection{Composite Finite Time Control Barrier Functions}
Consider two robots $R_1$ and $R_2$ as shown in the workspace in Fig~\ref{fig:ex_compfcbf}. Suppose $R_1$ is sensing information from $R_2$ and hence must always stay within the sensing radius of $R_2$. Suppose we have two regions of interest $A$, $B$ and the base $C$. Let $\mathcal{D}$ represent a corridor in the state space (denoted by the dotted lines in Fig~\ref{fig:ex_compfcbf}) where $R_1$ must maintain a very small distance of connectivity with $R_2$. This could represent, for example, an area with very poor network connectivity and hence the robots must resort to communication over small distances. Let the specification for the multi-agent system be given as $\phi = \Diamond(\pi_1^{A} \wedge \pi_{2}^{B}) \wedge \Diamond (\pi_{1}^{C} \wedge \pi_{2}^{C}) \wedge \Box \pi_{conn}$ where $\pi_{1}^{A}$ is the proposition that is true when $R_1$ is in $A$, $\pi_{2}^{B}$ is the proposition that is true when $R_2$ is in $B$, and $\pi_{conn}$ is the proposition that is true when the robots must maintain connectivity at all times. In other words, $R_1$ must visit $A$, $R_2$ must visit $B$ and then both must return to the base $C$. In addition, $R_1$ must always stay connected with $R_2$. The workspace is as shown in Fig~\ref{fig:ex_compfcbf}.
\begin{figure}[thpb]
\centering
    \includegraphics[width=7.5cm,height=5.5cm,scale = 1]{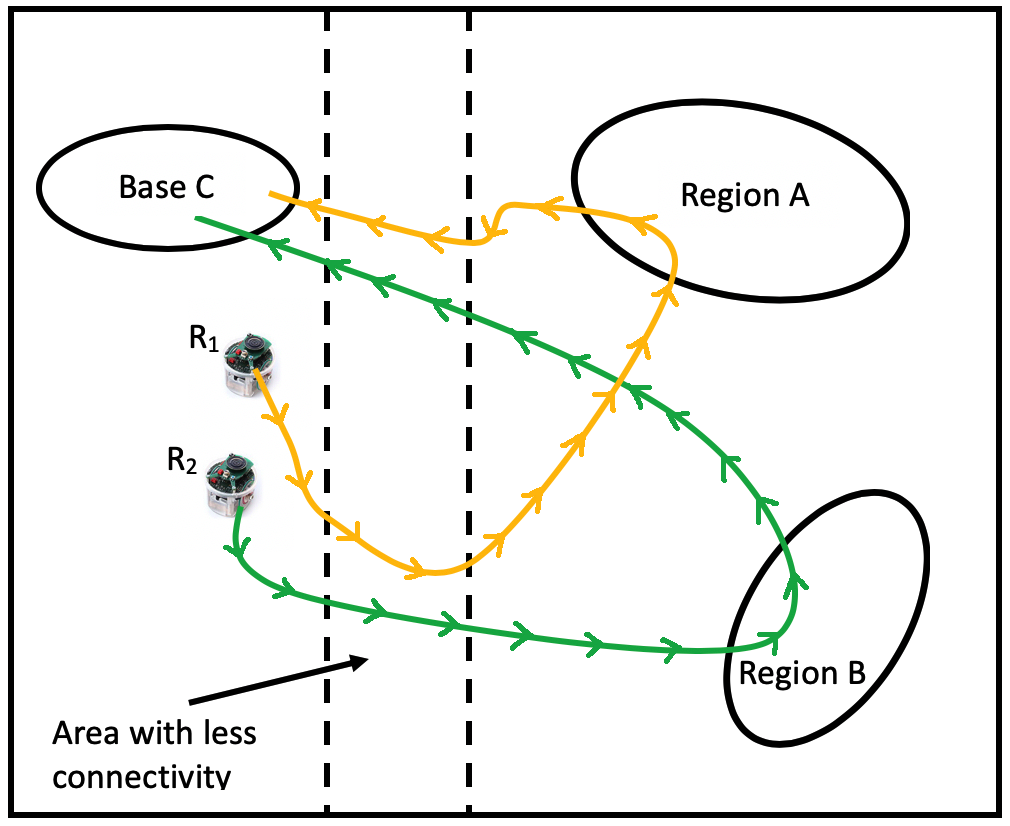}
    \caption{Representative trajectories for $R_1$ and $R_2$ that satisfy the specification $\phi = \Diamond(\pi_1^{A} \wedge \pi_{2}^{B}) \wedge \Diamond (\pi_{1}^{C} \wedge \pi_{2}^{C}) \wedge \Box \pi_{conn}$. The area with less connectivity is the corridor $\mathcal{D}$. Observe that $R_1$ and $R_2$ need to maintain a small distance of connectivity within the corridor $D$.}
    \label{fig:ex_compfcbf}
\end{figure}

By following the method proposed in \cite{li2018formally}, the QP that is to be solved is as follows:
\begin{equation}
\begin{aligned}
\label{individual_constraints}
& \underset{u \in \mathbb{R}^{4}}{\text{min}}
\quad ||u||_{2}^{2}\\
& \text{s.t \quad} L_{f}h_{A}(x_{1}) + L_{g}h_{A}(x_{1}) u \geq -\gamma \cdot \text{sign}(h_{A}(x_{1})) \cdot |h_{A}(x_{1})|^{\rho} \\
& \quad \quad L_{f}h_{B}(x_{2}) + L_{g}h_{B}(x_{2}) u \geq -\gamma \cdot \text{sign}(h_{B}(x_{2})) \cdot |h_{B}(x_{2})|^{\rho} \\
& \quad \quad L_{f}h_{conn}(x) + L_{g}h_{conn}(x) u \geq - \alpha(h_{conn}(x))
\end{aligned}
\end{equation}
where $\alpha$ is a locally Lipschitz extended class $\kappa$ function, $\gamma > 0$, $\rho \in [0,1)$, $x_1$ is the state of $R_1$, $x_2$ is the state of $R_2$, $x = \begin{bmatrix} x_1 & x_2\end{bmatrix}^{T}$ is the total state of the system, $h_{A}$ is the FCBF which represents $A$, $h_{B}$ is the FCBF which represents $B$, and $h_{conn}$ is the ZCBF which dictates the connectivity radius to be maintained by $R_1$ with $R_2$. Note that \cite{li2018formally} considers only reachability tasks and not more general LTL tasks. In addition, \cite{li2018formally} requires multiple reachability specifications to be encoded as separate constraints in a QP, as discussed above.

However, this QP becomes infeasible at the point when $R_1$ and $R_2$ reach the corridor $\mathcal{D}$. This is because the first constraint in \eqref{individual_constraints} dictates that $R_1$ make progress towards $A$, but the third constraint dictates that $R_1$ move closer to $R_2$ and hence move away from $A$. This leads to an empty solution space thus rendering the QP infeasible. This shows that the above formulation of encoding multiple reachability objectives as individual constraints is too restrictive. 

In light of the above scenario, we propose a method in which we compose multiple FCBFs. By ensuring that the total sum of the finite time barrier functions is always increasing, we can allow for decrease in the values of some of the individual barrier functions thereby allowing some robots to move away from their desired sets temporarily. This provides a larger solution space for the QP. This is formalized in the following theorem.

\begin{theorem}
\label{thm:comp_fcbf}
Consider a dynamical system in control-affine form as in \eqref{ControlAffineSystem}. Given $\Gamma \subset \mathcal{X}$ defined by a collection of $q\geq 1$ functions $\left\{ h_{i}(x) \right\}_{i = 1}^{q}$ such that $\Gamma = \bigcap\limits_{i = 1}^{q} \left\{ x \in \mathcal{X} \mid h_{i}(x) \geq 0 \right\}$ and for $ i = \left\{ 1, 2, 3, . . . , q' \right\}$ with $q' < q$, $h_{i}(x)$ is bounded i.e. $h_i(x) < M_i$ for all $x \in \mathcal{X}$, for $M_i > 0$.\footnote{If all the functions are bounded, then $q' = q$ and so we will have only \eqref{Constraint1} as a constraint in the QP $\forall i \in \left\{ 1,2,\dots, q \right\}$} If there exists a collection $\left\{ \alpha_{i} \right\}_{i = 1}^{q'}$ with $\alpha_{i} \in \mathbb{R}_{> 0}$, parameters $\gamma > 0$, $\rho \in [0,1)$ and a continuous controller $u(x)$ where $u : \mathcal{X} \rightarrow \mathbb{R}^{m}$, such that for all $x \in \mathcal{X}$
\begin{multline}
\label{Constraint1}
\sum\limits_{i = 1}^{q'} \bigg\{ \alpha_{i}(L_{f}h_{i}(x) + L_{g}h_{i}(x)u(x)) \bigg\} + \\
\gamma \cdot \text{sign} \bigg(\min \bigg\{ h_{1}(x), h_{2}(x),\dots, h_{q'}(x) \bigg\}\bigg) \geq 0
\end{multline}
\begin{multline}
\label{Constraint2}
L_{f}h_{i}(x) + L_{g}h_{i}(x)u(x) + \gamma \cdot \text{sign}({h_{i}(x)}) \cdot |h_{i}(x)|^{\rho} \geq 0 \\
\text{ $\forall$ i $\in \left\{ q'+1, \dots , q \right\}$}
\end{multline}
then under the feedback controller $u(x)$, for all initial conditions $x_{0} \in \mathcal{D}$, there exists $0 < T < \infty$ such that $x(T) \in \Gamma$.
\end{theorem}
\begin{proof}
By contradiction, suppose for some $x_0 \in \mathcal{X} \backslash \Gamma$ the control law $u(x)$ that satisfies \eqref{Constraint1} and \eqref{Constraint2} is such that there does not exist a finite time $0 < T < \infty$ so that $x(T) \in \Gamma$. In particular, then for all $t > 0$,  $min \bigg\{ h_{1}(x(t)), h_{2}(x(t)),\dots, h_{q}(x(t)) \bigg\} < 0$, where $x(t)$ is the solution to \eqref{ControlAffineSystem} initialized at $x(0)$ under the control law $u(x)$. By \eqref{Constraint2} for all $t > T_i  = \frac{|h_i(x_0)|^{1 - \rho}}{\gamma (1 - \rho)}$, we have $h_{i}(x(t)) \geq 0$ for all $i = \{ q'+1,\dots,q\}$ by Proposition 1. To that end, if we define $T' = \max\limits_{i = q'+1,\dots,q} \big\{ T_i \big\}$, then for all $t > T'$ we have, $\min \bigg\{ h_{1}(x(t)), h_{2}(x(t)),\dots, h_{q'}(x(t)) \bigg\} < 0$. In particular, observe that
\begin{align}
\label{const}
\frac{d}{dt} \sum\limits_{i = 1}^{q'}\bigg\{ \alpha_{i} h_{i}(x(t)) \bigg\} = \sum\limits_{i = 1}^{q'} \bigg\{ \alpha_{i}(L_{f}h_{i}(x) + L_{g}h_{i}(x)u(x)) \bigg\}
\end{align}
so that by integration of \eqref{const} using the fundamental theorem of calculus and \eqref{Constraint1}, we have
\begin{align*}
\sum\limits_{i = 1}^{q'}\bigg\{ \alpha_{i} h_{i}(x(t)) \bigg\} \geq \gamma \cdot (t-T') + \sum\limits_{i = 1}^{q'}\bigg\{ \alpha_{i} h_{i}(x(T')) \bigg\}.
\end{align*}
We observe that as $t \rightarrow \infty$, $\sum\limits_{i = 1}^{q'} \bigg\{ \alpha_i h_{i}(x(t))\bigg\} \rightarrow \infty$. But this is a contradiction since $h_{i}(x(t))$ for $i = \{ 1,2\dots,q'\}$ is bounded i.e. $\sum\limits_{i = 1}^{q'} \bigg\{ \alpha_i h_{i}(x(t))\bigg\} < \sum\limits_{i = 1}^{q'} \alpha_i M_i$. Thus, there exists a $0 < T < \infty$ such that $x(T) \in \bigcap\limits_{i = 1}^{q} \left\{ x \in \mathcal{X} \mid h_{i}(x) \geq 0 \right\}$.
\end{proof}

Theorem~\ref{thm:comp_fcbf} allows a system to reach an intersection of multiple regions in the state space using a single barrier certificate constraint. 
In contrast, \cite{li2018formally} proposes a more restrictive solution to the constrained reachability problem with desired level sets being individually defined by multiple FCBFs in a QP. In particular, 
\cite{li2018formally} proposes the set of control laws $\underline{\mathcal{U}}$ given by
\begin{align}
\nonumber &\underline{\mathcal{U}}(x)= \\
\nonumber &\{ u \in \mathbb{R}^{m} \mid  L_{f}h_{i}(x) + L_{g}h_{i}(x)u+ \gamma \cdot \text{sign}({h_{i}(x)}) \geq 0\\
\label{eq:other_form} &\hspace{2.1in}\forall i \in \left\{ 1, \dots , q \right\} \}.
\end{align}
Note that this is equivalent to taking $q' = 0$ in Theorem~\ref{thm:comp_fcbf}. Define
\begin{align}
\label{eq:compfcbf_controls}
\mathcal{U}(x) =\{ u \in \mathbb{R}^{m} \mid \text{\eqref{Constraint1} and \eqref{Constraint2} are satisfied} \},
\end{align}
and $\widehat{\Gamma} = \bigcup\limits_{i=1}^{q} \{ x \in \mathcal{X} \mid h_{i}(x) \geq 0\}$. We then formulate the following corollary.
\begin{cor}\label{cor:comp_fcbf}
Under the hypotheses of Theorem~\ref{thm:comp_fcbf} with $\sum\limits_{i=1}^{q'}\alpha_{i} \geq 1$, the set $\mathcal{U}(x)$ defined in \eqref{eq:compfcbf_controls} is a superset to the set $\underbar{$\mathcal{U}$}(x)$ defined in \eqref{eq:other_form}. That is, $\underbar{$\mathcal{U}$}(x) \subseteq \mathcal{U}(x)$ for all $x \in \mathcal{X} \backslash \widehat{\Gamma}$.
\end{cor}
\begin{proof}
Note that for $q' = 0$ and $q' = 1$, it follows that $\underline{\mathcal{U}}(x) = \mathcal{U}(x)$. For any $q' \in \{2,\ldots,q-1\}$, consider any $u(x) \in \underline{\mathcal{U}}(x)$ applied to the system \eqref{ControlAffineSystem}. Hence, we have that 
\begin{align*}
    L_{f}h_{i}(x) + L_{g}h_{i}(x) u(x) &\geq -\gamma \cdot \sign(h_{i}(x(t)))\geq \gamma \\
    \alpha_{i}(L_{f}h_{i}(x) + L_{g} & h_{i}(x)u(x)) \geq \alpha_{i} \cdot \gamma
\end{align*}
for all $i \in \{ 1,2,\ldots,q\}$ since $x \in \mathcal{X} \backslash \widehat{\Gamma}$. Summing the inequalities for the barrier functions from $i = 1, 2, \ldots, q'$, we get
\begin{align*}
    \sum\limits_{i=1}^{q'} \bigg\{ \alpha_{i}(L_{f}h_{i}(x) + L_{g}h_{i}(x) u(x)) \bigg\} \geq \sum\limits_{i=1}^{q'}\alpha_{i} \cdot \gamma \geq \gamma
\end{align*}
where the last inequality follows by assumption that $\sum\limits_{i=1}^{q'} \alpha_i \geq 1$. This implies $u(x)$ satisfies \eqref{Constraint1} since $\sign \bigg(\min \bigg\{ h_{1}(x), h_{2}(x),\dots, h_{q'}(x) \bigg\}\bigg) < 0$ for all $x \in \mathcal{X} \backslash \widehat{\Gamma}$, and $L_{f}h_{i}(x) + L_{g}h_{i}(x) u(x) \geq -\gamma \cdot \sign(h_{i}(x(t)))$ for all $i \in \{ q'+1, \ldots, q\}$ which implies $u(x)$ also satisfies \eqref{Constraint2}. Hence, $u(x) \in \mathcal{U}(x)$. Thus the corollary follows.
\end{proof}

Theorem~\ref{thm:comp_fcbf} allows for additional directions of evolution for the system state thereby resulting in a more relaxed approach to the finite time reachability problem. Corollary 1 provides an intuition regarding the key take-away from Theorem~\ref{thm:comp_fcbf}.

\subsection{Prioritization of Zeroing Control Barrier Functions}
In this subsection, we introduce a methodology for prioritizing different ZCBFs. In particular, our proposed formulation is similar to \cite{notomista2019optimal} where different tasks represented by multiple ZCBFs are prioritized for a multi-agent system. In particular, we stipulate that the performance objective dictated by the FCBFs must never be relaxed, and hence, they are encoded as \emph{hard} constraints (i.e. constraints which must never be violated), whereas the ZCBFs can be relaxed and hence, they are encoded as \emph{soft} constraints (i.e. constraints which could potentially be violated). Our proposed method is different from \cite{notomista2019optimal} in the sense that, in addition to the ZCBFs, we also incorporate FCBFs which are treated as hard constraints in the QP based controller.

Consider the following motivating example. Suppose we have a goal region $\mathcal{G} = \{ x \in \mathcal{X} | h_{\mathcal{G}}(x) \geq 0\}$ where $h_{\mathcal{G}} : \mathcal{X} \rightarrow \mathbb{R}$ is a FCBF, encapsulated by an obstacle $\mathcal{O} = \{ x \in \mathcal{X} | h_{\mathcal{O}}(x) \leq 0\}$ where $h_{\mathcal{O}} : \mathcal{X} \rightarrow \mathbb{R}$ is a ZCBF. Suppose the specification for the robot is $\phi = \Diamond \mathcal{G} \wedge \Box \neg \mathcal{O}$. By following the method proposed in existing works such as \cite{7857061}, \cite{7526486}, \cite{7989375}, \cite{li2018formally}, the QP that is to be solved is as follows:
\begin{equation}
\begin{aligned}
\label{qp}
& \underset{u \in \mathbb{R}^{m}}{\text{min}}
\quad ||u||_{2}^{2}\\
& \text{s.t \quad} L_{f}h_{\mathcal{G}}(x) + L_{g}h_{\mathcal{G}}(x) u \geq -\gamma \cdot \text{sign}(h_{\mathcal{G}}(x)) \cdot |h_{\mathcal{G}}(x)|^{\rho} \\
& \quad \quad L_{f}h_{\mathcal{O}}(x) + L_{g}h_{\mathcal{O}}(x) u \geq -\alpha(h_{\mathcal{O}}(x))
\end{aligned}
\end{equation}
where $\gamma > 0$, $\rho \in [0,1)$ and $\alpha$ is a locally Lipschitz extended class $\kappa$ function.

However, since the goal is encapsulated by the obstacle, the two constraints are in conflict and hence the QP will be infeasible. In order to tackle scenarios such as the one above, we propose a relaxed formulation of the QP similar to the one in \cite{xu2016robustness}, \cite{notomista2019optimal}.

Consider $p$ ZCBFs and $n$ FCBFs. Let $\mathcal{P}$ be the index sets for the zeroing barrier functions. Some or all of the ZCBFs may be in conflict with the composite FCBF. The generalized relaxed QP is of the form,
\begin{equation}
\begin{aligned}
\label{relaxed_qp}
& \underset{v = \begin{bmatrix} u^{T}, \epsilon_1, \dots, \epsilon_{p}\end{bmatrix}^{T} \in \mathbb{R}^{m+p}}{\text{min}}
\quad ||u||_{2}^{2} + \frac{1}{2}\Xi^{T}W\Xi \\
& \text{s.t \quad} \eqref{Constraint1} \text{ holds } \\
& \quad \quad L_{f}h_{i}(x) + L_{g}h_{i}(x) u(x) \geq -\alpha_{i}(h_{i}(x)) - \epsilon_i \text{ , }\forall i \in \mathcal{P}
\end{aligned}
\end{equation}
where $\Xi = \begin{bmatrix} \epsilon_1, \epsilon_2, \dots, \epsilon_p \end{bmatrix}^{T} \in \mathbb{R}^{p}$, $W \in \mathbb{R}^{p \times p}$ is a diagonal matrix with the diagonal elements as $(w_1, w_2, \dots, w_p)$ where $w_i \in \mathbb{R}_{>0}$ is a weight associated with the the slack variable $\epsilon_i$ for all $i \in \{ 1,2,\dots, p\}$, and $\alpha_i$ is a locally Lipschitz extended class $\kappa$ function. The weight matrix $W$ allows one to encode the notion of ``priority" for the barrier functions. For example, if the weight $w_i$ corresponding to the slack variable $\epsilon_i$ is large, then then $i^{th}$ ZCBF has higher priority over other constraints.

\begin{rmk}
Similar to the discussion in Remark 2 in \cite{xu2016robustness}, if the reachability and invariance constraints are not in conflict, then with an appropriate choice of the weight matrix $W$, we will have $\epsilon_i \approx 0$ for some $i \in \{ 1,2,\dots, q\}$. Also, note that we extend the formulation provided in \cite{xu2016robustness} from two constraints to multiple constraints.
\end{rmk}

The relaxed QP returns a control law that allows the system to reach the desired level set in a finite time while minimally violating the invariance constraints if there is a conflict with the FCBF. We present the following case study which uses the relaxation based controller \eqref{relaxed_qp}.
\subsubsection{Example}
Consider an omnidirectional robot with dynamics $\dot{x} = u$ 
%
where $x \in \mathcal{X} \subset \mathbb{R}^2$, and $u \in \mathbb{R}^2$. Let $\mathcal{D} \subset \mathcal{X}$ be a compact domain in the state space. The workspace is as shown in Fig~\ref{fig:casestudy2_traj}. Suppose we have two unsafe regions $A$ and $B$ and a goal region $C$. Let $C$ be contained within $A$ and $B$. Suppose the specification to be satisfied by the robot is $\phi = \Diamond C \wedge \Box (\neg A \wedge \neg B)$. From Fig~\ref{fig:casestudy2_traj}, we observe that satisfaction of $\phi$ is impossible without entering the regions $A$ or $B$. However, suppose that region $A$ has greater priority than region $B$ and hence violation of $B$ is allowed to some extent.

With this additional flexibility, we can employ the proposed QP as in \eqref{relaxed_qp} with the weights $w_A \in \mathbb{R}_{> 0}$ set to be a large value and $w_B \in \mathbb{R}_{> 0}$ set to be a small value.  We then solve \eqref{relaxed_qp} which gives us a family of trajectories (depending on the values of the weights $w_A$ and $w_B$) of the robot as shown in Fig~\ref{fig:casestudy2_traj}. Observe that with different weights $w_A$ and $w_B$ for the regions in the QP, we obtain a different trajectory. This allows one to also encode the notion of priority in the QP.
\begin{figure}
    \includegraphics[width=0.95\columnwidth
    ]{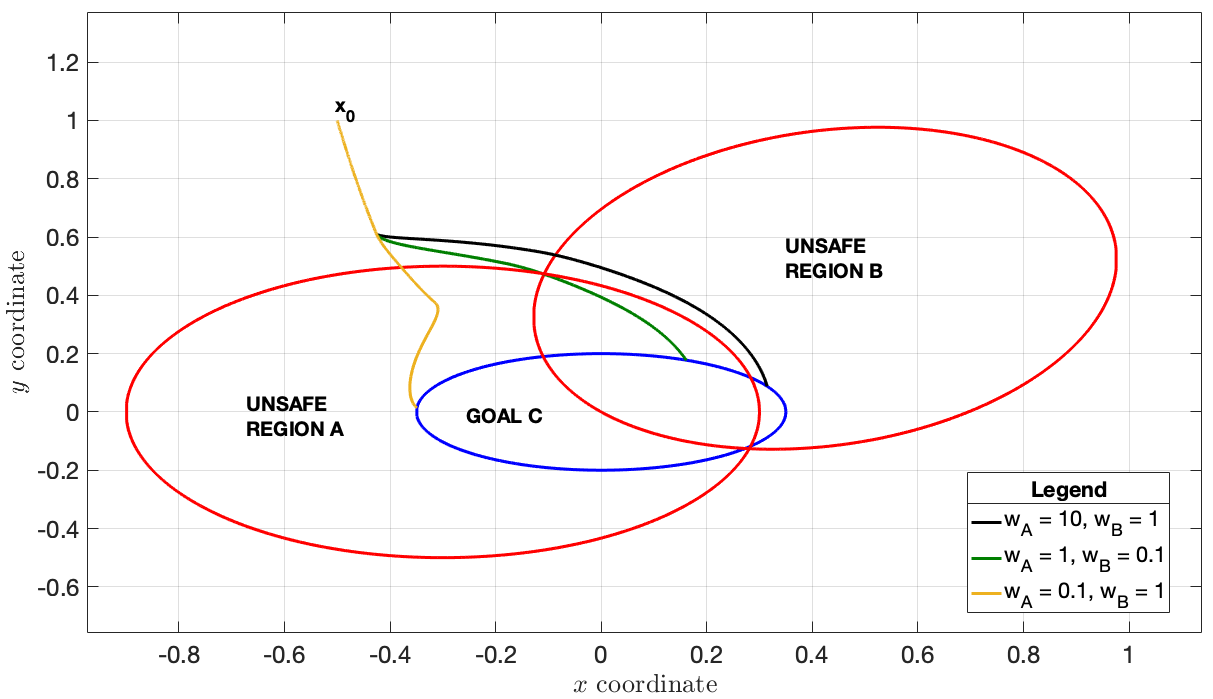}
      \caption{A family of trajectories for the robot generated by the relaxed QP \eqref{relaxed_qp}. By changing the values of the entries in the weight matrix $W$, one can encode the notion of priority for different regions in the state space as can be seen from the various trajectories.}
      \label{fig:casestudy2_traj}
\end{figure}

\section{Synthesis and Analysis of Quadratic Program based Controller}
\label{sec:formal_guar}
In this section, we detail the theoretical framework which provides formal guarantees that the quadratic program (QP) based controller indeed produces a system trajectory that satisfies the given specification. We also describe the methodology to synthesize the barrier funtion based QP controller given an $LTL_{robotic}$ specification.

\subsection{Lasso Type Constrained Reachability Objectives}
It is well established that if there exists a trace that satisfies a specification belonging to the fragment $LTL_{\text{robotic}}$, then there exists a trace which satisfies the specification in \emph{lasso} or \emph{prefix-suffix} form (\cite{baier2008principles}, pp 272), where a trace $\sigma$ is in lasso form if it is comprised of a finite horizon prefix $\sigma_{pre}$ and a finite horizon suffix $\sigma_{\text{suff}}$ that is repeated infinitely often. Both $\sigma_{pre}$ and $\sigma_{\text{suff}}$ are finite sequences of sets of atomic propositions such that the trace $\sigma$ is equal to the prefix followed by an infinite repetition of the suffix. Such a lasso-type trace is denoted as $\sigma = \sigma_{pre}(\sigma_{\text{suff}})^{\omega}$, where $\omega$ denotes infinite repetition. Atomic propositions of a continuous time dynamical system are subsets of the domain, and, hence, it is possible to interpret such lasso traces as sequences of constrained reachability problems in lasso form, which forms the basis of our control synthesis methodology. This is formalized in the following definitions.

\begin{defn}[Constrained reachability objective]
\label{def:reach_obj}
Given a target set $\Gamma \subset \mathcal{X}$ and a safety set $\Sigma \subset \mathcal{X}$, the \emph{constrained reachability objective}, denoted by $R(\Sigma, \Gamma)$, is defined as the reachability problem to be solved so that the state of the system reaches the set $\Gamma$ in finite time while remaining in $\Sigma$ until it reaches $\Gamma$. \QEDB
\end{defn}

The constrained reachability objective for a system \eqref{ControlAffineSystem} is solved from a given initial condition in $\Sigma$ if a control policy is found which drives the state of the system to $\Gamma$ while remaining in $\Sigma$ until it reaches $\Gamma$. For example, a reachability objective denoted by $R(B, A)$ signifies that the system must reach region $A$ in finite time while staying in region $B$ until it reaches region $A$. The constrained reachability objective implies finding a control policy that solves the above objective successfully.

\begin{defn}[Lasso Type Constrained Reachability Sequence]
\label{lassodefn}
A \emph{lasso-type constrained reachability sequence} is a sequence of constrained reachability objectives in lasso form such that each subsequent safety set is compatible with the prior goal set. That is, a lasso-type constrained reachability sequence has the form
  \begin{align}
    \label{reachlasso}
\mathcal{R}_{lasso} = \bigg(R_1R_2\ldots R_p\bigg)\bigg(R_{p+1},R_{p+2}\ldots R_{p+\ell}\bigg)^\omega \text{,}
  \end{align}
where $p > 0$, $\ell\geq 1$, and each $R_j=R(\Sigma_j, \Gamma_j)$ for some $\Gamma_j,\Sigma_j\subset \mathcal{X}$ satisfying $\Gamma_j\subseteq \Sigma_{j+1}$ for all $j\in\{1,2,\ldots,p+\ell-1\}$ and $\Gamma_{p+\ell} \subseteq \Sigma_{p+1}$. The sequence $(R_1R_2\ldots R_p)$ is a finite horizon prefix objective and $(R_{p+1},R_{p+2}\ldots R_{p+\ell})$ is a finite horizon suffix objective that is repeated infinitely often. \QEDB
\end{defn}
The lasso-type constrained reachability sequence is considered feasible if each constituent reachability objective is solved successfully in sequence. For example, consider the task specification, ``The robot must first visit region $A$, then region $B$, while avoiding the obstacle $C$. The lasso-type reachability sequence that satisfies this task is given by
\begin{align*}
    \mathcal{R}_{\text{lasso}} = \bigg( R_1 R_2 \bigg) \bigg(R_3 \bigg)^{\omega}
\end{align*}
where $R_1 = R(\text{C}, \text{A})$, $R_2 = R(\text{C}, \text{B})$ and $R_3 = R(\text{C}, \emptyset)$. Note that if $p = 0$, then the finite prefix has length zero and the lasso sequence is then given by
  \begin{align}
  \label{eq:zero_prefix_lasso}
    \mathcal{R}_{lasso} = \bigg(R_{1},R_{2}\ldots R_{\ell}\bigg)^\omega \text{.}
  \end{align}

By the preceding discussion, if there exists a trace that satisfies a given $LTL_{robotic}$ specification, then there exists a lasso-type constrained reachability sequence which, if feasible, guarantees that the system satisfies  the $LTL_{robotic}$ specification. One can view the lasso type reachability sequence as a bridge between the $LTL_{robotic}$ specification and the set based approach of our proposed controller.

\subsection{Construction Of Lasso-type Reachability Sequence}
Consider a $LTL_{robotic}$ specification $\phi$ as in \eqref{spec}. Given $\phi$, our first objective is to generate the lasso-type constrained reachability sequence of the form \eqref{reachlasso}.
\begin{defn}[Lasso Template]
\label{template}
Given a $LTL_{robotic}$ specification $\phi$, a \emph{lasso template} is an enumeration of the form
\begin{align}
\mathcal{O}_2 : \{ 1,2,\dots, k\} \rightarrow \mathcal{I}_2\\
\mathcal{O}_3 : \{ 1,2,\dots, \ell \} \rightarrow \mathcal{I}_3
\end{align}
where the index sets $\mathcal{I}_2$ and $\mathcal{I}_3$ are as per Definition~\ref{def:LTLfrag} and $k = |\mathcal{I}_2|$ and $\ell = \max\{ |\mathcal{I}_3|, 1\}$. \QEDB
\end{defn}

Note that it is computationally straightforward to obtain \emph{some} lasso template simply by arbitrarily enumerating the elements of the index sets $\mathcal{I}_2$ and $\mathcal{I}_3$. 

A lasso-type reachability sequence of the form \eqref{reachlasso} or \eqref{eq:zero_prefix_lasso} is constructed using Algorithm 1. The $LTL_{\text{robotic}}$ specification, and the lasso template $\mathcal{O}_{2}$ and $\mathcal{O}_{3}$ are the inputs to Algorithm 1. The output is a lasso-type constrained reachability sequence of the form \eqref{reachlasso} or \eqref{eq:zero_prefix_lasso}.

\subsection{Synthesis of Quadratic Program based Controller}
\begin{algorithm}
\caption{Lasso-type Reachability Sequence Generator}
 \hspace*{\algorithmicindent} \textbf{Input : $\phi$, $\mathcal{O}_2$, $\mathcal{O}_3$} \\
 \hspace*{\algorithmicindent} \textbf{Output: $\mathcal{R}_{lasso}$}
 \begin{algorithmic}[1]
 \IF{$J_4 \neq \emptyset$}
 \STATE $p \leftarrow k+1$
 \IF{$k \neq 0$}
 \STATE $\Gamma_i = \llbracket \psi_{2}^{\mathcal{O}_{2}(i)} \rrbracket$ for all $i = 1,2,\dots,p-1$
 \STATE $\Sigma_i = \llbracket \psi_1 \rrbracket$ for all $i = 1,2,\dots, p-1$
 \ENDIF
 \STATE $\Gamma_p = \llbracket \psi_{4} \rrbracket$
 \STATE $\Sigma_p = \llbracket \psi_{1} \rrbracket$
 \STATE $\Gamma_{p+i} = \llbracket \psi_{3}^{\mathcal{O}_{3}(i)} \rrbracket$ for all $i = 1,2,\dots, \ell$
 \IF{$J_1 = \emptyset$}
 \STATE $\Sigma_{p+i} = \llbracket \psi_4 \rrbracket$ for all $i =1,2,\dots,\ell$ 
 \ELSE
 \STATE $\Sigma_{p+i} = \llbracket \psi_1 \rrbracket \cap \llbracket \psi_4 \rrbracket$ for all $i = 1,2,\dots,\ell$
 \ENDIF
 \STATE $R_{i} = R_{i}(\Sigma_i, \Gamma_i)$ for all $i = 1,2,\dots, p+\ell$
 \RETURN $\mathcal{R}_{lasso}$ as in \eqref{reachlasso}
 \ELSE
 \STATE $p \leftarrow k$
 \IF{$p \neq 0$}
 \STATE $\Gamma_i = \llbracket \psi_{2}^{\mathcal{O}_{2}(i)} \rrbracket$ for all $i = 1,2,\dots,p$
 \STATE $\Gamma_{p+i} = \llbracket \psi_{3}^{\mathcal{O}_{3}(i)} \rrbracket$ for all $i = 1,2,\dots, \ell$
 \STATE $\Sigma_i = \llbracket \psi_1 \rrbracket$ for all $i = 1,2,\dots, p+\ell$
 \STATE $R_{i} = R_{i}(\Sigma_i, \Gamma_i)$ for all $i = 1,2,\dots, p+\ell$
 \RETURN $\mathcal{R}_{lasso}$ as in \eqref{reachlasso}
 \ELSE
  \STATE $\Gamma_{p+i} = \llbracket \psi_{3}^{\mathcal{O}_{3}(i)} \rrbracket$ for all $i = 1,2,\dots, \ell$
 \STATE $\Sigma_{p+i} = \llbracket \psi_1 \rrbracket$ for all $i = 1,2,\dots,\ell$
 \STATE $R_{i} = R_{i}(\Sigma_i, \Gamma_i)$ for all $i = 1,2,\dots, p+\ell$
 \RETURN $\mathcal{R}_{lasso}$ as in \eqref{eq:zero_prefix_lasso}
 \ENDIF
 \ENDIF
 \end{algorithmic}
\end{algorithm}

We next encode the reachability objectives as finite time and zeroing control barrier functions in a QP. This is described in Algorithm 2. Each $\Gamma_i$ is encoded with FCBFs with \eqref{eq:fcbf_controls} or \eqref{eq:compfcbf_controls} as constraint(s) whereas each $\Sigma_i$ is encoded with ZCBFs with \eqref{eq:zcbf_controls} as constraint(s) in the QP. The designer is free to choose a locally Lipschitz $\alpha$ function for \eqref{eq:zcbf_controls}. In order to solve a particular reachability objective $R_{i}(\Sigma_i, \Gamma_i)$ where $i \in \{1,2,\dots,n\}$, we solve a QP as in \eqref{intro_qp}. Note that solving a QP in real time is typically done in a few milliseconds, and hence Algorithm 2 is amenable to real time implementation on robotic platforms.
\begin{algorithm}
\caption{Quadratic Program based Controller}
 \hspace*{\algorithmicindent} \textbf{Input : $\mathcal{R}_{lasso}$}
\begin{algorithmic}[1]
\IF{$p \neq 0$}
\FOR{$i = 1,2,\dots, p$}
\STATE Encode $\Gamma_i$ with FCBFs
\STATE Encode $\Sigma_i$ with ZCBFs
\WHILE{$x \notin \Gamma_i$}
\STATE Solve $R(\Sigma_i, \Gamma_i)$ as in \eqref{intro_qp}
\ENDWHILE
\ENDFOR
\ENDIF
\WHILE{true}
\FOR{$i = p+1, \dots, p+\ell$}
\STATE Encode $\Gamma_i$ with FCBFs
\STATE Encode $\Sigma_i$ with ZCBFs
\WHILE{$x \notin \Gamma_i$}
\STATE Solve $R(\Sigma_i, \Gamma_i)$ as in \eqref{intro_qp}
\ENDWHILE
\ENDFOR
\ENDWHILE
\end{algorithmic}
\end{algorithm}

\subsection{Analysis Of Trajectory Generated by QP Controller}
Observe that there is a one-to-one correspondence between elements of $P(\Pi_{aug})$ and subsets of $\Pi$. Let $\iota : 2^{\Pi} \rightarrow P(\Pi_{aug}) \subset 2^{\Pi_{aug}}$ be the canonical bijective mapping for a subset $\sigma \in 2^{\Pi}$ with the corresponding mapping $\iota(\sigma) \in P(\Pi_{aug})$ given by, 
\begin{equation}
\pi \in \sigma \iff \pi \in \iota(\sigma) \text{ and } \pi \not \in \sigma \iff \overline{\pi} \in \iota(\sigma). 
\end{equation}
For notational convenience, we do not explicitly differentiate between a subset $\sigma\subset \Pi$ and its mapping $\iota(\sigma)\in P(\Pi_{aug})$.

Given Algorithm 2, we now provide formal guarantees which prove that the QP from Algorithm 2 indeed produces a system trajectory which satisfies the system specification.

\begin{defn}[Descendant]
\label{def:desc}
Given a $LTL_{robotic}$ specification $\phi$ with a lasso template $\mathcal{O}_2$ and $\mathcal{O}_3$ as in Definition \ref{template}, a \emph{descendant} of the lasso template is any infinite length sequence of the form
\begin{multline}
\label{desc}
\sigma = \bigg\{ \sigma_{1,1} \sigma_{1,2} \dots \sigma_{1, n_1}\bigg\} \bigg\{ \sigma_{2,1} \sigma_{2,2} \dots \sigma_{2, n_2}\bigg\} \dots \\
\bigg\{ \sigma_{p,1} \sigma_{p,2} \dots \sigma_{p,n_{p}}\bigg\} \dots \text{,}
\end{multline}
where $\sigma_{i, j} \in P(\Pi_{aug})$ for all $i = 1,2,\dots$, $j = 1,2,\dots, n_i$ and
\begin{enumerate}
\item $J_1 \subseteq \sigma_{i, j}$ for all $i \in \{ 1,2,\dots, p\}$ and for all $j \in \{ 1,2,\dots, n_i\}$
\item $J_{2}^{\mathcal{O}_{2}(i)} \subseteq \sigma_{\mathcal{O}_{2}(i),n_{\mathcal{O}_{2}(i)}}$ for all $i \in \{ 1,2,\dots, k\}$
\item $J_4 \subseteq \sigma_{p, n_{p}}$
\item $J_{3}^{\mathcal{O}_{3}(i)} \subseteq \sigma_{m, n_m}$ where $m = p + d \ell + \mathcal{O}_{3}(i)$ for all $d \in \{ 0,1,2 \dots\}$ and for all $i \in \{ 1,2,\dots,\ell\}$
\item $J_1 \cup J_4 \subseteq \sigma_{i, j}$ for all $i \in \{ p+1, \dots\}$ and for all $j \in \{ 1,2,\dots, n_i\}$. \QEDB
\end{enumerate}
\end{defn}

Intuitively, a descendant $\sigma$ of a given template is a sequence of atomic propositions visited by the system such that it respects the safety sets $\Sigma_i$ and also reaches the target sets $\Gamma_i$ in a finite time for all $i \in \{ 1,2,\dots,p+l\}$. Consider the example discussed before, where a robot must first visit region $A$, then region $B$ while avoiding the obstacle $C$. The lasso template for this task is $\mathcal{O}_{2}(1) = A$, $\mathcal{O}_{2}(2) = B$. Given this template, one valid instantiation of the descendant \eqref{desc} is
\begin{align*}
    \sigma = \bigg\{ \{\bar{A},\bar{B},\bar{C}\}\{A,\bar{B},\bar{C}\} \bigg\} \bigg\{ \{\bar{A},\bar{B},\bar{C}\}\{\bar{A}, B,\bar{C}\} \bigg\}\bigg\{ \{ \bar{A}, B, \bar{C}\}\bigg\}^{\omega} , 
\end{align*}
which satisfies the conditions of Definition~\ref{def:desc}.

In \eqref{desc}, each set $\sigma_i = \{ \sigma_{i,1} \sigma_{i,2} \dots \sigma_{i,n_i}\}$ corresponds to the $i^{th}$ constrained reachability objective in the lasso sequence \eqref{reachlasso} or \eqref{eq:zero_prefix_lasso} and the set $\sigma_p = \{ \sigma_{p,1}\sigma_{p,2}\dots\sigma_{p,n_p}\}$ is the last constrained reachability objective in the finite prefix part of the lasso sequence after which the sequence switches to the suffix.

\begin{prop}
\label{prop:descendant}
Given a lasso template as in Definition \ref{template} for a $LTL_{robotic}$ specification $\phi$ as in \eqref{spec}, any descendant $\sigma$ of this template is such that $\sigma \models \phi$.
\end{prop}
\begin{proof}
Let $\phi = \phi_{globe} \wedge \phi_{reach} \wedge \phi_{rec} \wedge \phi_{act}$ be a specification as in \eqref{spec}. Let $\mathcal{O}_2$ and $\mathcal{O}_3$ be a lasso template for the specification. Let $\sigma$ be a descendant of the lasso template as in Definition \ref{def:desc}.

We provide a proof by construction by considering four individual cases for the specification $\phi$. Then, since conjunction preserves the results from these cases (Fig 5.2, pp 236 \cite{baier2008principles}), we combine them to provide a proof for the entire fragment of LTL.

Case 1: Suppose $\phi = \phi_{globe} = \Box \psi_1$ for $\psi_1 = \bigwedge\limits_{m=1}^{n} \pi_m$, where $\pi_m \in \Pi_{aug}$. Thus we have $J_1 = \{\pi_1, \dots, \pi_n \}$, $J_{2}^{\mathcal{O}_{2}(i)} = \{ \emptyset \}$ for all $i \in \{ 1,2,\dots, k\}$, $J_{3}^{\mathcal{O}_{3}(i)} = \{ \emptyset \}$ for all $i \in \{ 1,2,\dots, \ell\}$ and $J_4 = \{ \emptyset \}$. A descendant trace of the template is as per Definition \ref{def:desc}. Thus, from condition 1 in Definition~\ref{def:desc}
, we observe that $J_1 = \{\pi_1, \dots, \pi_n \} \subseteq \sigma_{i,j}$ for all $i \in \{1,2,\dots\}$ and for all $j \in \{ 1,2, \dots, n_i\}$. Hence, we can conclude that $\sigma \models \phi_{globe}$.

Case 2: Suppose $\phi = \phi_{act} = \Diamond \Box \psi_4$ for $\psi_4 = \bigwedge\limits_{m=1}^{n} \pi_m$ where $\pi_{m} \in \Pi_{aug}$. Thus we have $J_1 = \{ \emptyset \}$, $J_{2}^{\mathcal{O}_{2}(i)} = \{ \emptyset \}$ for all $i \in \{ 1,2,\dots, k\}$, $J_{3}^{\mathcal{O}_{3}(i)} = \{ \emptyset \}$ for all $i \in \{ 1,2,\dots, \ell\}$ and $J_4 = \{\pi_1, \dots, \pi_n \}$. A descendant trace of the template has a closed form expression as in Definition \ref{def:desc}. Thus, from condition 3 in Definition~\ref{def:desc}, we have $J_4 \subseteq \sigma_{p, n_p}$, and from condition 5 in Definition~\ref{def:desc}, we observe that $J_4 = \{\pi_1, \dots, \pi_n \} \subseteq \sigma_{i,j}$ for all $i \in \{p+1,p+2,\dots\}$ and for all $j \in \{ 1,2, \dots, n_i\}$. Hence, we can conclude that $\sigma \models \phi_{act}$.

Case 3: Suppose $\phi = \phi_{reach} = \bigwedge\limits_{j \in \mathcal{I}_{2}} \Diamond \psi_{2}^{j}$. Thus we have $J_1 = \{ \emptyset \}$, $J_{3}^{\mathcal{O}_{3}(i)} = \{ \emptyset \}$ for all $i \in \{ 1,2,\dots, \ell\}$ and $J_4 = \{ \emptyset \}$. A descendant trace of the template has a closed form expression as in Definition \ref{def:desc}. Thus, from condition 2 in the definition, we observe that $J_{2}^{\mathcal{O}_{2}(m)} = \{\pi_1, \dots, \pi_n \} \subseteq \sigma_{\mathcal{O}_{2}(m),n_{\mathcal{O}_{2}(m)}}$ for all $m \in \{ 1,2,\dots, k\}$. Hence, we can conclude that $\sigma \models \phi_{reach}$.

Case 4: Suppose $\phi = \phi_{rec} = \bigwedge\limits_{j \in \mathcal{I}_{3}} \Box \Diamond \psi_{3}^{j}$.  Thus we have $J_1 = \{ \emptyset \}$, $J_{2}^{\mathcal{O}_{2}(i)} = \{ \emptyset \}$ for all $i \in \{ 1,2,\dots, k\}$ and $J_4 = \{ \emptyset \}$. A descendant trace of the template has a closed form expression as in Definition \ref{def:desc}. Thus, from condition 4 in the definition, we observe that $J_{3}^{\mathcal{O}_{3}(q)} = \{\pi_1, \dots, \pi_n \} \subseteq \sigma_{m, n_m}$ for all $m = p + dl + \mathcal{O}_{3}(q)$, for all $d \in \{ 0,1,2 \dots\}$ and for all $q \in \{ 1,2,\dots, \ell\}$. Hence, we can conclude that $\sigma \models \phi_{rec}$.

Thus, by combining the results from Cases 1, 2, 3 and 4 with conjunction \cite{baier2008principles} (Fig 5.2, pp 236 \cite{baier2008principles}), we can conclude that $\sigma$ satisfies $\phi = \phi_{globe} \wedge \phi_{reach} \wedge \phi_{rec} \wedge \phi_{act}$. That is, $\sigma \models \phi$.
\end{proof}

Next we state Theorem~\ref{thm:formal_guarantee} which provides a theoretical guarantee that if Algorithm 2 is feasible, then the trace of the resulting system trajectory satisfies the specification.

\begin{theorem}
\label{thm:formal_guarantee}
Given a $LTL_{robotic}$ specification $\phi$ and a lasso template $\mathcal{O}_2$ and $\mathcal{O}_3$ as in Definition \ref{template}, let $\mathcal{R}_{lasso}$ be the lasso-type constrained reachability sequence as in \eqref{reachlasso} generated from Algorithm 1. If Algorithm 2 is feasible, then the trace of the system trajectory $x(t)$ satisfies $\phi$.
\end{theorem}
\begin{proof}
As per Algorithm 2, each $\Sigma_i$ is encoded as constraint(s) with ZCBFs for all $i \in \{ 1,2,\dots, p+\ell\}$. From Proposition~\ref{prop:zcbf}, this guarantees forward invariance of the atomic propositions that need to remain true or need to remain false. Since the QP from Algorithm 2 is feasible, conditions 1 and 5 from Definition \ref{def:desc} are satisfied. Since each $\Gamma_i$ is encoded as constraint(s) with FCBFs for all $i \in \{ 1,2,\dots, p+\ell\}$, from Theorem~\ref{thm:comp_fcbf} we can guarantee finite time convergence to atomic propositions that need to be reached in finite time. This satisfies conditions 2, 3 and 4 of Definition~\ref{def:desc}. Thus, all conditions in Definition \ref{def:desc} are satisfied. Since the QP is feasible, we conclude that Algorithm 2 generates a descendant $\sigma$ of the lasso template.

From Proposition~\ref{prop:descendant}, we know that given a lasso template, any descendant $\sigma$ of the lasso template is such that it satisfies the specification. From the previous analysis, we know that the QP from Algorithm 2 produces a descendant of the lasso template. The mapping $\iota$ being bijective and combining Proposition~\ref{prop:descendant} with the previous analysis, we can conclude that QP from Algorithm 2 produces a trace of the trajectory of the system that satisfies the given specification. 
That is, $\iota^{-1}(\sigma)  = \sigma \models \phi$.
\end{proof}

Note that while Algorithm 2 and Theorem~\ref{thm:formal_guarantee} assume that the QP \eqref{intro_qp} is feasible, one can always use the relaxed QP \eqref{relaxed_qp} for feasibility. In that case, although feasibility of the controller is more likely, Theorem~\ref{thm:formal_guarantee} may no longer hold since the relaxation parameters $\epsilon$ can be non-zero so that the corresponding atomic propositions are no longer satisfied. However, such a situation is not considered in this paper.

\section{Case Study}
\label{sec:case_study}
In this section, we provide a case study that details the barrier functions based QP framework which synthesizes a system trajectory that satisfies the specification. This case study was implemented in the Robotarium multi-robot testbed at Georgia Tech \cite{pickem2017robotarium}. The Robotarium consists of differential drive mobile robots which can be programmed using either MATLAB or Python.

Consider a team of three robots: one surveillance robot ($\mathtt{R}_{3}$) and two attack robots ($\mathtt{R}_{1}$ and $\mathtt{R}_{2}$). The surveillance robot needs to collect information regarding the position of two targets, and then return back to the base. Once the information has been relayed to the base by the surveillance robot, the attack robots must visit the targets infinitely often. In addition to this, the attack robots must stay connected with each other at all times, and all the robots must avoid a danger zone where they can be attacked.

Let $\mathcal{D} \subset \mathbb{R}^{2}$ be the workspace for each robot and let $\mathcal{D} \times \mathcal{D} \times \mathcal{D} \subset \mathbb{R}^{6}$ be the domain of the three robot system with regions $\mathcal{A} = \left\{ \emph{A}, \emph{B}, \emph{C}, \emph{O} \right\}$. The dynamics for each agent $i \in \{ 1,2,3\}$ is
\begin{align*}
    \begin{bmatrix} \dot{p}^{1}_{i} \\ \dot{p}^{2}_{i} \\ \dot{\phi_{i}} \end{bmatrix} &= \begin{bmatrix} \cos(\phi_{i}) & 0 \\ \sin(\phi_{i}) & 0 \\ 0 & 1\end{bmatrix} \begin{bmatrix} v_i \\ \omega_i \end{bmatrix} \text{,}
\end{align*}
where $p_{i}^{1} \in \mathbb{R}$ and $p_{i}^{2} \in \mathbb{R}$ represent the position of the robot, $\phi_{i} \in (-\pi,\pi]$ represents its orientation, $v_{i} \in \mathbb{R}$ and $\omega_{i} \in \mathbb{R}$ are the linear and angular velocity inputs to the robot respectively. Denote $x_{i} = \begin{bmatrix} p_{i}^{1} & p_{2}^{i} & \phi_{i} \end{bmatrix}^{T}$, and $\bar{x}_{i} = \begin{bmatrix} p_{i}^{1} & p_{2}^{i} \end{bmatrix}^{T}$. For implementation purposes in the Robotarium and theoretical reasons associated with the unicycle robot as discussed in \cite{paul_nonsmooth}, we use the NID technique discussed in \cite{NID} to control the differential drive robots as a single integrator model. The NID technique allows for control over both input to a differential drive robots as discussed in \cite{paul_nonsmooth}.

Target 1 is labelled as $A$, target 2 is labelled as $B$ the base is labelled as $C$, and $O$ is the danger zone (obstacle). The set of atomic propositions is given by $\Pi=\{\pi^r_i,\overline{\pi}^r_i\} \cup \{\pi^{conn},\overline{\pi}^{conn}\}$ for all $i \in \{ 1,2,3\}$ and $r \in \{A,B,C,O \}$. The regions \emph{A}, \emph{B}, \emph{C} are defined as $\llbracket \pi_{i}^{r} \rrbracket = \left\{ x \in \mathcal{D}^{3} | h_{r}(x_{i}) \geq 0 \right\}$ for all $r \in \{ A,B,C,O\}$ and for all $i \in \{1,2,3 \}$. For each $\llbracket \pi_{i}^{r} \rrbracket$ with $i \in \{ 1,2,3\}$, $r \in \{A,B,C,O \}$, let
\begin{equation}
\label{pie}
\pi_{i}^{r} = \left\{
        \begin{array}{ll}
            1 & \quad \text{if } \bar{x}_i \in  r\\
            0 & \quad \text{otherwise.}
        \end{array}
    \right.
\end{equation}

This means $\pi_{i}^{r} = 1$ if and only if agent $i$ is in region $r$.
\color{black}
The additional connectivity constraint that must be maintained by $\mathtt{R}_{1}$ and $\mathtt{R}_{2}$ is given as $h_{conn}(x) \geq 0$ where
\begin{equation}
\label{distbarrier}
h_{conn}(x) = d_{conn}^{2}(x) - ||\bar{x}_{2} - \bar{x}_{1}||^{2},
\end{equation}
where $d_{conn} : \mathcal{D} \times \mathcal{D} \times \mathcal{D} \rightarrow \mathbb{R}$ is the connectivity distance between the two agents that needs to be maintained, and $|| \bar{x}_{2} - \bar{x}_{1}||$ is the inter-agent distance. We consider
\begin{equation}
\label{conn_rule}
d_{conn}^{2}(x)  = (p_{2}^{1} + \delta_{1})^{2} + \delta_{2},
\end{equation}
where $\delta_{1}$ and $\delta_{2}$ are constants. The connectivity set corresponding to the proposition $\pi^{conn}$ is defined as $\llbracket \pi^{conn} \rrbracket = \{ x \in \mathcal{D}^{3} | h_{conn}(x) \geq 0\}$. Such a constraint captures a situation in which the robots have poor connectivity in certain areas of the workspace, which requires them to maintain a closer distance with each other. In areas where the robots have strong connectivity, they are free to maintain a larger distance from each other.

The $LTL_{robotic}$ specification for the task described previously is
\begin{equation}
\label{eq:ltl_spec}
\begin{split}
\phi = (\Diamond \pi_{3}^{A} \wedge \Diamond \pi_{3}^{B} &\wedge \Diamond \pi_{3}^{C}) \wedge \Box \Diamond (\pi_{1}^{A} \wedge \pi_{2}^{B}) \wedge \Box \Diamond (\pi_{1}^{C} \wedge \pi_{2}^{C}) \\
&\wedge \Box ( \pi^{conn} \wedge \neg \pi_{1}^{O} \wedge \neg \pi_{2}^{O} \wedge \neg \pi_{3}^{O} ) \text{.}
\end{split}
\end{equation}

From the formalism in Definition \ref{lassodefn} and Algorithm 1, we obtain the lasso-type constrained reachability objective,
\begin{flalign}
\mathcal{R}_{lasso} = \bigg( R_{1}(\Sigma_1, \Gamma_1) R_{2}(\Sigma_2, \Gamma_2) R_{3}(\Sigma_3, \Gamma_3) \bigg) \nonumber \\
\bigg( R_{4}(\Sigma_4, \Gamma_4) R_{5}(\Sigma_5, \Gamma_5) \bigg)^{\omega}
\end{flalign}
where $\Sigma_i = \llbracket \pi^{conn} \rrbracket \cap \llbracket \overline{\pi_{1}}^{O} \rrbracket \cap \llbracket \overline{\pi_{2}}^{O} \rrbracket \cap \llbracket \overline{\pi_{3}}^{O} \rrbracket$ for $i = 1,2,3,4,5$, $\Gamma_1 = \llbracket \pi_{3}^{A} \rrbracket$, $\Gamma_2 = \llbracket \pi_{3}^{B} \rrbracket$, $\Gamma_3 = \llbracket \pi_{3}^{C} \rrbracket$, $\Gamma_4 = \llbracket \pi_{1}^{A} \rrbracket \cap \llbracket \pi_{2}^{B} \rrbracket$, $\Gamma_5 = \llbracket \pi_{1}^{C} \rrbracket \cap \llbracket \pi_{2}^{C} \rrbracket$.
\begin{figure}[thpb]
    \begin{center}
    \includegraphics[width=0.9\columnwidth]{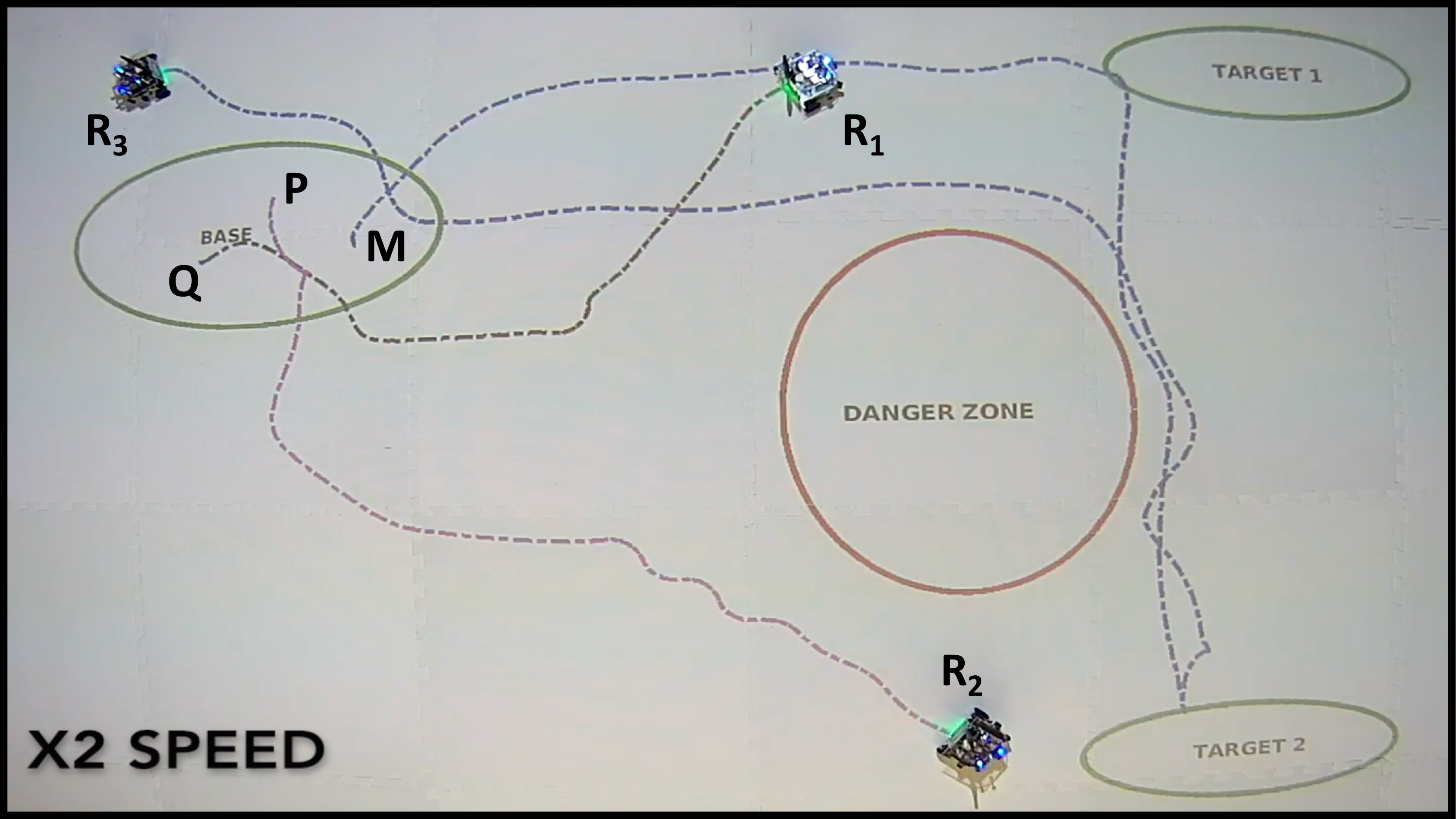}
      \caption{A still shot of the trajectories for the robots $\mathtt{R}_1$, $\mathtt{R}_2$ and $\mathtt{R}_3$ for the specification $\phi$ as in \eqref{eq:ltl_spec}. Observe that $\mathtt{R}_1$ moves temporarily away from target 1 temporarily in order to satisfy the connectivity constraint dictated by \eqref{distbarrier} and \eqref{conn_rule}, but Theorem~\ref{thm:comp_fcbf} results in feasible solutions at those points. From the figure, we observe that the robots maintain connectivity and avoid the danger zone at all times.}
      \label{fig:cs1_traj}
    \end{center}
\end{figure}

Next, we use Algorithm 2 to generate the pointwise controller for the system. Each reachability objective $R_{i}(\Sigma_i, \Gamma_i)$ for all $i \in \{ 1,2,3,4,5\}$ is encoded as a QP and is solved sequentially.
In particular, if $\mathcal{U}_{i}(x)$ is the set of feasible control laws that satisfies all the constraints for each reachability objective, then for all $i = \{ 1,2,3,4,5\}$ the QP solved is given by,
\begin{equation}
\begin{aligned}
& \underset{u \in \mathbb{R}^{6}}{\text{min}}
\quad ||u||_{2}^{2}\\
& \text{s.t \quad} u \in \mathcal{U}_{i}(x) \text{.}
\end{aligned}
\end{equation}

From Theorem~\ref{thm:formal_guarantee} we conclude that these trajectories indeed satisfy the specification $\phi$. The switching between the current reachability objective to the next is automatic. It occurs when the state of the system reaches the desired set of states. That is, the switching from reachability objective $i$ to objective $i+1$ occurs when $x \in \Gamma_i$ for all $i \in \{ 1,2,3,4\}$.

As is discussed in \cite{paul_nonsmooth}, the computation complexity of strongly convex QPs, a class to which the proposed controller belongs to, is $O((m + d)^{3})$ where $m$ is the number of decision variables, and $d$ is the number of constraints in the QP. This is the complexity of most commonly used solvers. In our experiments, we observed computation times in the order of 10-15 milliseconds for each QP. At all times, $\mathtt{R}_{1}$ and $\mathtt{R}_{2}$ stay connected as per the distance dictated by \eqref{distbarrier} and avoid the danger zone, as seen in Fig~\ref{fig:cs1_traj}. Thus, we see that by solving this sequence of constrained reachability objectives, the multi-agent system satisfies the specification. Fig~\ref{fig:cs1_traj} is a still shot of the experiment conducted on the Robotarium testbed at Georgia Tech \cite{pickem2017robotarium} \footnote{Video of Robotarium experiment -\url{https://youtu.be/EK1Zxcg-eSE}} \footnote{Code for Robotarium experiment- \url{https://bit.ly/37QkOBS}}

\section{Concluding Remarks}
\label{sec:conclusion}
In this paper, we provided a framework for the control of mobile robotic systems with control affine dynamics. In particular, we used control barrier functions and temporal logic as the tools to develop this framework. First, we discussed issues regarding feasibility of the QP based controller. We provided a new method to compose multiple FCBFs in order to obtain a larger feasible solution set as compared to existing methods in literature. We also proposed a modified QP based controller which prioritizes different ZCBFs. Second, we developed a fully automated framework which synthesizes a barrier function based controller given a specification. Last, we provided formal guarantees that the QP based controller generates a system trajectory that satisfies the given specification.

\bibliographystyle{IEEEtran}

\bibliography{bibtex}

\begin{thebibliography}{10}
\providecommand{\url}[1]{#1}
\csname url@samestyle\endcsname
\providecommand{\newblock}{\relax}
\providecommand{\bibinfo}[2]{#2}
\providecommand{\BIBentrySTDinterwordspacing}{\spaceskip=0pt\relax}
\providecommand{\BIBentryALTinterwordstretchfactor}{4}
\providecommand{\BIBentryALTinterwordspacing}{\spaceskip=\fontdimen2\font plus
\BIBentryALTinterwordstretchfactor\fontdimen3\font minus
  \fontdimen4\font\relax}
\providecommand{\BIBforeignlanguage}[2]{{%
\expandafter\ifx\csname l@#1\endcsname\relax
\typeout{** WARNING: IEEEtran.bst: No hyphenation pattern has been}%
\typeout{** loaded for the language `#1'. Using the pattern for}%
\typeout{** the default language instead.}%
\else
\language=\csname l@#1\endcsname
\fi
#2}}
\providecommand{\BIBdecl}{\relax}
\BIBdecl

\bibitem{murray_manipulators}
S.~{Chinchali}, S.~C. {Livingston}, U.~{Topcu}, J.~W. {Burdick}, and R.~M.
  {Murray}, ``Towards formal synthesis of reactive controllers for dexterous
  robotic manipulation,'' in \emph{2012 IEEE International Conference on
  Robotics and Automation}, May 2012, pp. 5183--5189.

\bibitem{pers_assistants}
L.~P. {Kaelbling} and T.~{Lozano-Pérez}, ``Hierarchical task and motion
  planning in the now,'' in \emph{2011 IEEE International Conference on
  Robotics and Automation}, May 2011, pp. 1470--1477.

\bibitem{7989375}
L.~{Wang}, A.~D. {Ames}, and M.~{Egerstedt}, ``Safe certificate-based maneuvers
  for teams of quadrotors using differential flatness,'' in \emph{2017 IEEE
  International Conference on Robotics and Automation (ICRA)}, May 2017, pp.
  3293--3298.

\bibitem{DHS}
\BIBentryALTinterwordspacing
``Trustworthy cyber infrastructure for the power grid (tcip-g),'' Mar 2019.
  [Online]. Available:
  \url{https://www.dhs.gov/science-and-technology/trustworthy-cyber-infrastructure-power-grid-tcip-g}
\BIBentrySTDinterwordspacing

\bibitem{simon_2019}
\BIBentryALTinterwordspacing
M.~Simon, ``Your first look inside amazon's robot warehouse of tomorrow,'' Jun
  2019. [Online]. Available:
  \url{https://www.wired.com/story/amazon-warehouse-robots/}
\BIBentrySTDinterwordspacing

\bibitem{SCI}
\BIBentryALTinterwordspacing
 [Online]. Available:
  \url{https://www.dhs.gov/cisa/critical-infrastructure-sectors}
\BIBentrySTDinterwordspacing

\bibitem{forsgren2002interior}
A.~Forsgren, P.~E. Gill, and M.~H. Wright, ``Interior methods for nonlinear
  optimization,'' \emph{SIAM review}, vol.~44, no.~4, pp. 525--597, 2002.

\bibitem{ames2014control}
A.~D. {Ames}, J.~W. {Grizzle}, and P.~{Tabuada}, ``Control barrier function
  based quadratic programs with application to adaptive cruise control,'' in
  \emph{53rd IEEE Conference on Decision and Control}, Dec 2014, pp.
  6271--6278.

\bibitem{xu2016robustness}
X.~Xu, P.~Tabuada, J.~W. Grizzle, and A.~D. Ames, ``Robustness of control
  barrier functions for safety critical control,'' \emph{arXiv preprint
  arXiv:1612.01554}, 2016.

\bibitem{7857061}
L.~Wang, A.~D. Ames, and M.~Egerstedt, ``Safety barrier certificates for
  collisions-free multirobot systems,'' \emph{IEEE Transactions on Robotics},
  vol.~33, no.~3, pp. 661--674, June 2017.

\bibitem{7526486}
L.~Wang, A.~Ames, and M.~Egerstedt, ``Safety barrier certificates for
  heterogeneous multi-robot systems,'' in \emph{2016 American Control
  Conference (ACC)}, July 2016, pp. 5213--5218.

\bibitem{li2018formally}
A.~Li, L.~Wang, P.~Pierpaoli, and M.~Egerstedt, ``Formally correct composition
  of coordinated behaviors using control barrier certificates,'' in \emph{2018
  IEEE/RSJ International Conference on Intelligent Robots and Systems
  (IROS)}.\hskip 1em plus 0.5em minus 0.4em\relax IEEE, 2018, pp. 3723--3729.

\bibitem{mohit_cdc18}
M.~{Srinivasan}, S.~{Coogan}, and M.~{Egerstedt}, ``Control of multi-agent
  systems with finite time control barrier certificates and temporal logic,''
  in \emph{2018 IEEE Conference on Decision and Control (CDC)}, Dec 2018, pp.
  1991--1996.

\bibitem{sanfelice}
M.~{Maghenem} and R.~G. {Sanfelice}, ``Barrier function certificates for
  forward invariance in hybrid inclusions,'' in \emph{2018 IEEE Conference on
  Decision and Control (CDC)}, Dec 2018, pp. 759--764.

\bibitem{baier2008principles}
C.~Baier and J.-P. Katoen, \emph{Principles of model checking}, 2008.

\bibitem{alur_disc_hybsys}
R.~{Alur}, T.~A. {Henzinger}, G.~{Lafferriere}, and G.~J. {Pappas}, ``Discrete
  abstractions of hybrid systems,'' \emph{Proceedings of the IEEE}, vol.~88,
  no.~7, pp. 971--984, July 2000.

\bibitem{belta_rect}
C.~{Belta} and L.~C. G. J.~M. {Habets}, ``Controlling a class of nonlinear
  systems on rectangles,'' \emph{IEEE Transactions on Automatic Control},
  vol.~51, no.~11, pp. 1749--1759, Nov 2006.

\bibitem{fully_auto_belta}
M.~{Kloetzer} and C.~{Belta}, ``A fully automated framework for control of
  linear systems from temporal logic specifications,'' \emph{IEEE Transactions
  on Automatic Control}, vol.~53, no.~1, pp. 287--297, Feb 2008.

\bibitem{receding_topcu}
T.~{Wongpiromsarn}, U.~{Topcu}, and R.~M. {Murray}, ``Receding horizon temporal
  logic planning,'' \emph{IEEE Transactions on Automatic Control}, vol.~57,
  no.~11, pp. 2817--2830, Nov 2012.

\bibitem{vardi_complex_goals}
A.~{Bhatia}, M.~R. {Maly}, L.~E. {Kavraki}, and M.~Y. {Vardi}, ``Motion
  planning with complex goals,'' \emph{IEEE Robotics Automation Magazine},
  vol.~18, no.~3, pp. 55--64, Sep. 2011.

\bibitem{fainekos2009temporal}
G.~E. Fainekos, A.~Girard, H.~Kress-Gazit, and G.~J. Pappas, ``Temporal logic
  motion planning for dynamic robots,'' \emph{Automatica}, vol.~45, no.~2, pp.
  343--352, 2009.

\bibitem{lindemann2019control}
L.~Lindemann and D.~V. Dimarogonas, ``Control barrier functions for signal
  temporal logic tasks,'' \emph{IEEE control systems letters}, vol.~3, no.~1,
  pp. 96--101, 2019.

\bibitem{lindemann_coupled}
L.~{Lindemann} and D.~V. {Dimarogonas}, ``Decentralized control barrier
  functions for coupled multi-agent systems under signal temporal logic
  tasks,'' in \emph{2019 18th European Control Conference (ECC)}, June 2019,
  pp. 89--94.

\bibitem{Raman}
V.~{Raman}, A.~{Donzé}, M.~{Maasoumy}, R.~M. {Murray},
  A.~{Sangiovanni-Vincentelli}, and S.~A. {Seshia}, ``Model predictive control
  with signal temporal logic specifications,'' in \emph{53rd IEEE Conference on
  Decision and Control}, Dec 2014, pp. 81--87.

\bibitem{belta2019formal}
C.~Belta and S.~Sadraddini, ``Formal methods for control synthesis: An
  optimization perspective,'' \emph{Annual Review of Control, Robotics, and
  Autonomous Systems}, 2019.

\bibitem{Liu}
Z.~{Liu}, B.~{Wu}, J.~{Dai}, and H.~{Lin}, ``Distributed communication-aware
  motion planning for multi-agent systems from stl and spatel specifications,''
  in \emph{2017 IEEE 56th Annual Conference on Decision and Control (CDC)}, Dec
  2017, pp. 4452--4457.

\bibitem{Farhani_shrinking}
S.~S. {Farahani}, R.~{Majumdar}, V.~S. {Prabhu}, and S.~E.~Z. {Soudjani},
  ``Shrinking horizon model predictive control with chance-constrained signal
  temporal logic specifications,'' in \emph{2017 American Control Conference
  (ACC)}, May 2017, pp. 1740--1746.

\bibitem{aksaray_belta_q_learning}
D.~{Aksaray}, A.~{Jones}, Z.~{Kong}, M.~{Schwager}, and C.~{Belta},
  ``Q-learning for robust satisfaction of signal temporal logic
  specifications,'' in \emph{2016 IEEE 55th Conference on Decision and Control
  (CDC)}, Dec 2016, pp. 6565--6570.

\bibitem{muniraj2018enforcing}
D.~Muniraj, K.~G. Vamvoudakis, and M.~Farhood, ``Enforcing signal temporal
  logic specifications in multi-agent adversarial environments: A deep
  q-learning approach,'' in \emph{2018 IEEE Conference on Decision and Control
  (CDC)}.\hskip 1em plus 0.5em minus 0.4em\relax IEEE, 2018, pp. 4141--4146.

\bibitem{dimos_control_guided}
\BIBentryALTinterwordspacing
P.~V{\'{a}}rnai and D.~V. Dimarogonas, ``Prescribed performance control guided
  policy improvement for satisfying signal temporal logic tasks,'' \emph{CoRR},
  vol. abs/1903.04340, 2019. [Online]. Available:
  \url{http://arxiv.org/abs/1903.04340}
\BIBentrySTDinterwordspacing

\bibitem{lindemann2018decentralized}
L.~Lindemann and D.~V. Dimarogonas, ``Decentralized robust control of coupled
  multi-agent systems under local signal temporal logic tasks,'' in \emph{2018
  Annual American Control Conference (ACC)}.\hskip 1em plus 0.5em minus
  0.4em\relax IEEE, 2018, pp. 1567--1573.

\bibitem{pant_mangharam_drones}
Y.~V. {Pant}, H.~{Abbas}, R.~A. {Quaye}, and R.~{Mangharam}, ``Fly-by-logic:
  Control of multi-drone fleets with temporal logic objectives,'' in \emph{2018
  ACM/IEEE 9th International Conference on Cyber-Physical Systems (ICCPS)},
  April 2018, pp. 186--197.

\bibitem{high_rel_deg_nguyen}
Q.~{Nguyen} and K.~{Sreenath}, ``Exponential control barrier functions for
  enforcing high relative-degree safety-critical constraints,'' in \emph{2016
  American Control Conference (ACC)}, July 2016, pp. 322--328.

\bibitem{xiao2019control_rel_deg}
W.~Xiao and C.~Belta, ``Control barrier functions for systems with high
  relative degree,'' 2019.

\bibitem{notomista2019optimal}
G.~Notomista, S.~Mayya, S.~Hutchinson, and M.~Egerstedt, ``An optimal task
  allocation strategy for heterogeneous multi-robot systems,'' \emph{arXiv
  preprint arXiv:1903.08641}, 2019.

\bibitem{khalil2002nonlinear}
H.~K. Khalil, \emph{Nonlinear systems}, vol.~3.

\bibitem{guo_ltl_planning}
{Meng Guo}, K.~H. {Johansson}, and D.~V. {Dimarogonas}, ``Revising motion
  planning under linear temporal logic specifications in partially known
  workspaces,'' in \emph{2013 IEEE International Conference on Robotics and
  Automation}, May 2013, pp. 5025--5032.

\bibitem{wolff2013efficient}
E.~M. {Wolff}, U.~{Topcu}, and R.~M. {Murray}, ``Efficient reactive controller
  synthesis for a fragment of linear temporal logic,'' in \emph{2013 IEEE
  International Conference on Robotics and Automation}, May 2013, pp.
  5033--5040.

\bibitem{hadas_reactive_planning}
H.~{Kress-Gazit}, G.~E. {Fainekos}, and G.~J. {Pappas}, ``Temporal-logic-based
  reactive mission and motion planning,'' \emph{IEEE Transactions on Robotics},
  vol.~25, no.~6, pp. 1370--1381, Dec 2009.

\bibitem{wongpiromsarn2016automata}
T.~Wongpiromsarn, U.~Topcu, and A.~Lamperski, ``Automata theory meets barrier
  certificates: Temporal logic verification of nonlinear systems,'' \emph{IEEE
  Transactions on Automatic Control}, vol.~61, no.~11, pp. 3344--3355, 2016.

\bibitem{7782377}
A.~D. Ames, X.~Xu, J.~W. Grizzle, and P.~Tabuada, ``Control barrier function
  based quadratic programs for safety critical systems,'' \emph{IEEE
  Transactions on Automatic Control}, vol.~62, no.~8, pp. 3861--3876, Aug 2017.

\bibitem{pickem2017robotarium}
D.~{Pickem}, P.~{Glotfelter}, L.~{Wang}, M.~{Mote}, A.~{Ames}, E.~{Feron}, and
  M.~{Egerstedt}, ``The robotarium: A remotely accessible swarm robotics
  research testbed,'' in \emph{2017 IEEE International Conference on Robotics
  and Automation (ICRA)}, May 2017, pp. 1699--1706.

\bibitem{paul_nonsmooth}
P.~{Glotfelter}, I.~{Buckley}, and M.~{Egerstedt}, ``Hybrid nonsmooth barrier
  functions with applications to provably safe and composable collision
  avoidance for robotic systems,'' \emph{IEEE Robotics and Automation Letters},
  vol.~4, no.~2, pp. 1303--1310, April 2019.

\bibitem{NID}
R.~{Olfati-Saber}, ``Near-identity diffeomorphisms and exponential
  $\epsilon$-tracking and $\epsilon$-stabilization of first-order nonholonomic
  {SE}(2) vehicles,'' in \emph{Proceedings of the 2002 American Control
  Conference (IEEE Cat. No.CH37301)}, vol.~6, May 2002, pp. 4690--4695 vol.6.

\end{thebibliography}

\begin{IEEEbiography}[{\includegraphics[width=1in,height=1.25in,clip,keepaspectratio]{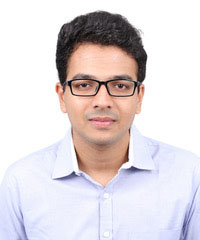}}]{Mohit Srinivasan} is currently pursuing his PhD degree in Electrical and Computer Engineering (ECE) from Georgia Institute of Technology, Atlanta, Georgia, USA. He earned his Master of Science in Electrical and Computer Engineering from Georgia Tech and Bachelor of Technology in Electrial Engineering from Veermata Jijabai Technological Institute (V.J.T.I), Mumbai, India. He was previously a research intern at Landis+Gyr, Atlanta. His research interests are primarily in multi-agent systems, control theory, and robotics.
\end{IEEEbiography}

\begin{IEEEbiography}[{\includegraphics[width=1in,height=1.25in,clip,keepaspectratio]{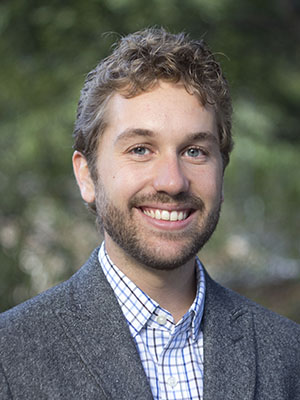}}]{Samuel Coogan} is an Assistant Professor at Georgia Tech in the School of Electrical and Computer Engineering and the School of Civil and Environmental Engineering. Prior to joining Georgia Tech in July 2017, he was an assistant professor in the Electrical Engineering Department at UCLA from 2015-2017. He received the B.S. degree in Electrical Engineering from Georgia Tech and the M.S. and Ph.D. degrees in Electrical Engineering from the University of California, Berkeley. His research is in the area of dynamical systems and autonomy and focuses on developing scalable tools for verification and control of networked, cyber-physical systems with an emphasis on transportation systems. He received a CAREER award from NSF in 2018, a Young Investigator Award from the Air Force Office of Scientific Research in 2018, and the Outstanding Paper Award for the IEEE Transactions on Control of Network Systems in 2017.
\end{IEEEbiography}

\end{document}